\documentclass[fleqn,10pt]{wlscirep}
\usepackage[utf8]{inputenc}
\usepackage[T1]{fontenc}
\usepackage{enumitem}

\definecolor{citecolor}{HTML}{0071bc}
\definecolor{citered}{HTML}{8b0000}
\usepackage{hyperref}
\usepackage{url}

\usepackage{booktabs}

\usepackage{amsmath,amsfonts,bm}

\def\eqref#1{equation~\ref{#1}}

\def\1{\bm{1}}

\DeclareMathAlphabet{\mathsfit}{\encodingdefault}{\sfdefault}{m}{sl}
\SetMathAlphabet{\mathsfit}{bold}{\encodingdefault}{\sfdefault}{bx}{n}

\usepackage{graphicx}
\usepackage{float}
\usepackage{amsfonts,amsmath,amssymb,amsthm}
\usepackage{mathtools}

\newtheorem{theorem}{Theorem}

\usepackage{cleveref}
\usepackage{multicol, multirow}
\usepackage{makecell}
\usepackage{caption}
\usepackage{subfigure}
\usepackage{adjustbox}
\usepackage{enumitem}
\usepackage{wrapfig}
\usepackage{algorithm}
\usepackage{afterpage}

\usepackage{xcolor}
\usepackage{booktabs}
\usepackage{algpseudocode}
\usepackage{mathrsfs}
\usepackage{dsfont}
\usepackage{bm}
\usepackage{subcaption}
\usepackage{bbm}
\usepackage{accents}
\usepackage{orcidlink}
\usepackage{nicefrac}

\usepackage{cleveref}

\usepackage{xcolor}
\newcommand{\model}[1]{ProteinDT}
\newcommand{\dataset}[1]{SwissProtCLAP}
\newcommand{\ProteinCLAP}[1]{ProteinCLAP}
\newcommand{\ProteinFacilitator}[1]{ProteinFacilitator}
\newcommand{\ProteinSDE}[1]{ProteinDiff}

\definecolor{ForestGreen}{RGB}{34,139,34}

\renewcommand{\thefootnote}{\fnsymbol{footnote}}
\usepackage{lmodern}
\usepackage{amssymb,amsmath}

\title{Valid Property-\textbf{E}nhanced \textbf{C}ontrastive Learning for \textbf{T}argeted \textbf{O}ptimization \& \textbf{R}esampling for Novel Drug Design}

\author[1,$*$,$\dagger$]{Amartya Banerjee \orcidlink{0000-0002-7998-3390}}
\author[2,5,$\dagger$]{Somnath Kar \orcidlink{0009-0004-8702-4335}}
\author[3,4,$\dagger$]{Anirban Pal
\orcidlink{0009-0001-7518-7661}}
\author[4,$*$]{Debabrata Maiti \orcidlink{0000-0001-8353-1306}}

\affil[1]{Department of Computer Science, University of North Carolina at Chapel Hill, NC 27599, USA}
\affil[2]{Radiation Medicine Centre, Bhabha Atomic Research Centre, Parel, Mumbai 400012, India}
\affil[3]{IITB-Monash Research Academy, IIT Bombay, Powai, Mumbai 400076, India}
\affil[4]{Department of Chemistry, Indian Institute of Technology Bombay, Powai, Mumbai 400076, India}
\affil[5]{Homi Bhabha National Institute, Anushaktinagar, Mumbai 400094, India}

\begin{abstract} \\
Efficiently steering generative models toward pharmacologically relevant regions of chemical space remains a major obstacle in molecular drug discovery in the low-data regime. We present \textbf{VECTOR+:} \textbf{V}alid‑property‑\textbf{E}nhanced \textbf{C}ontrastive Learning for \textbf{T}argeted \textbf{O}ptimization and \textbf{R}esampling, a contrastive‑learning and latent‑sampling framework that couples property‑guided representation learning with controllable molecule generation. VECTOR+ can be efficiently applied to both regression and classification tasks and enables interpretable, data-efficient exploration of functional chemical space in low-data setting. We demonstrate the method on two distinct datasets: a curated PD-L1 inhibitor set (296 compounds with experimentally determined IC\textsubscript{50} values) and an available receptor kinase inhibitor set (2,056 molecules categorized by binding mode). Despite limited training data, VECTOR+ successfully generates novel, synthetically tractable candidates. Against PD-L1 (PDB ID: 5J89), 100 of 8,374 generated molecules surpass a docking threshold of -$15.0$ kcal/mol, with the top candidate scoring -$17.6$ kcal/mol outperforming known inhibitors (best docking score: -$15.4$ kcal/mol). The best-performing molecules retain the conserved biphenyl pharmacophore while introducing novel chemical motifs. 250 ns molecular dynamics simulations confirm their binding stability (ligand RMSD < $2.5$ \r{A}).  VECTOR+ generalizes effectively to kinase inhibitors as well, producing compounds with superior docking scores compared to established drugs such as brigatinib and sorafenib. Benchmarking against state-of-the-art generative models, such as JT-VAE and MolGPT, across multiple metrics (docking score, novelty, uniqueness, Tanimoto similarity) demonstrates the superior performance of our method. Together, these results position our work as a robust and extensible approach for property-conditioned molecular design in low-data regime, bridging contrastive learning and generative modeling for reproducible, AI-accelerated discovery. 
\end{abstract}

\begin{document}
\flushbottom
\maketitle

\footnotetext[1]{Corresponding authors:
  \href{mailto:amartya1@cs.unc.edu}{amartya1@cs.unc.edu}; 
  \href{mailto:dmaiti@chem.iitb.ac.in}{dmaiti@chem.iitb.ac.in}}
\footnotetext[2]{These authors contributed equally.}

\renewcommand{\thefootnote}{\arabic{footnote}}
\setcounter{footnote}{0}

\thispagestyle{empty}

\section*{Introduction}

The process of discovering new therapeutics remains one of the most expensive, time-consuming, and uncertain endeavors in modern science, with costs exceeding US\$2.8 billion and a duration of over 12 years for the development of a novel drug \cite{wouters2020estimated}. Despite decades of progress in chemistry, biology, and screening technologies, the core workflow for small-molecule drug discovery is still dominated by empirical and iterative exploration. Machine learning (ML) has emerged as a transformative tool in modern drug discovery, revolutionizing traditional approaches through data-driven decision-making \cite{dara2022machine,chen2018rise,gupta2021artificial}. The field has evolved significantly from high-throughput screening and Computer-Aided Drug Design (CADD) to the integration of artificial intelligence (AI) and ML technologies. Today, researchers across the globe are leveraging ML models at various stages of the drug discovery pipeline, including target identification and validation \cite{biswas2020artificial}, virtual screening, hit and lead discovery, lead optimization, quantitative structure–activity relationship (QSAR) modeling \cite{pitt2025real}, and predictive toxicology \cite{yang2023application,sinha2023review,tran2023artificial}. 

More recently, the advent of generative AI has further accelerated the progress of de novo molecular design. These approaches aim to algorithmically generate novel compounds with optimized pharmacological and physicochemical properties, thereby reducing discovery timelines and resource requirements \cite{gangwal2024unlocking}. Among the most prominent architectures explored are variational autoencoders (VAEs) \cite{kingma2013auto}, generative adversarial networks (GANs) \cite{goodfellow2014generative}, and generative pretrained transformers (GPTs) \cite{haroon2023generative}, which have demonstrated promising capabilities for property-guided molecular generation and scaffold innovation.

Despite recent progress in generative AI for molecular design, the application of these models to domain-specific problems in chemistry and biology remains highly constrained by limited data availability \cite{dou2023machine}. Most real-world challenges in these fields involve only a few hundred to a few thousand data points. In contrast, current generative frameworks such as VAEs, GANs, and others are typically trained on massive, general-purpose databases like ZINC, ChEMBL, or PubChem, which contain millions of diverse compounds across a wide spectrum of chemical and biological applications \cite{parvatikar2023artificial,tingle2023zinc}. Consequently, when applied to a narrow, low-data regimes, these models often perform poorly, failing to capture domain-specific structure–function relationships \cite{van2024deep}. This highlights a critical need for new generative approaches that are explicitly designed to operate effectively in data-scarce, problem-specific settings. One such application space where this need is particularly acute is cancer-targeted drug discovery, which is a longstanding priority for both academic researchers and pharmaceutical industries. 

Cancer is the second leading cause of death after cardiovascular diseases \cite{zaorsky2017causes,bray2021ever}. New cancer cases and cancer-related deaths are increasing at a steep rate with each passing year.  Chemotherapy remains the widely used modality in the treatment of cancer, along with surgery, immunotherapy, and radiation therapy \cite{zugazagoitia2016current,wang2018combining}. Although chemotherapeutic drugs are used around the world in cancer therapy, due to their non-selectivity and specificity, they result in severe side effects \cite{schirrmacher2019chemotherapy,van2022chemotherapy}. 

This underscores the critical importance of designing target-specific anticancer therapies with improved efficacy and safety profiles. Among the most promising targets are the Programmed Death-Ligand 1 (PD-L1) and protein kinase systems, both of which exhibit aberrant expression across a wide spectrum of malignancies \cite{wang2016pd,yu2020pd,mahoney2015combination,yi2021regulation,yamaoka2018receptor,tomuleasa2024therapeutic,hsu2016role}. Tumor cells evade immune surveillance by exploiting checkpoint pathways, notably the PD-1/PD-L1 axis, where PD-L1 binding to PD-1 leads to T cell dysfunction and tumor progression. Blocking this interaction restores T cell activity and enhances antitumor immunity \cite{lin2024regulatory}. The kinase receptor system, on the other hand, particularly receptor tyrosine kinases (RTKs), contributes to tumor growth by transmitting signals that drive cell proliferation, survival, migration, and angiogenesis \cite{du2018mechanisms}. Hence, the design of inhibitors against these targets helps to inhibit cancer progression and metastasis.

In this work, we present a generative modeling framework tailored for the design of new small-molecule inhibitors targeting PD-L1 and kinase receptors. Our full code and data is available at \url{https://github.com/amartya21/vector-drug-design.git}.

\paragraph{Our major contributions are as follows:}
\begin{itemize}

    \item We propose a new method VECTOR+ which addresses the critical challenge of data scarcity by structuring the chemical latent space according to biological function. This enables meaningful representation learning from limited datasets, a common bottleneck in drug discovery. Furthermore, we ground our approach in a classical result by leveraging a fundamental principle of information theory that moment-matching Gaussians uniquely minimize forward KL divergence. This principle provides a strong rationale for our use of GMM-based sampling, validating VECTOR+ as a mathematically principled framework for latent space modeling.

    \item We validate our method on PD-L1 and Kinase inhibitor datasets. This leads to a high yield of novel candidates and has produced a PD-L1 inhibitor with a significantly high docking score of $-17.6$ kcal/mol, surpassing known compounds. For kinase inhibitors, starting with just 47 allosteric inhibitor training molecules, we successfully generated 2,500 novel, drug-like compounds with high validity and uniqueness.

    \item The proposed framework is broadly generalizable and can be readily extended to accelerate drug discovery across diverse therapeutic targets and molecular generation domains. While the method demonstrates strong performance in extremely low-data regimes, the framework is equally applicable when larger training datasets are available, where we expect it to further benefit from richer class structure and more stable surrogate fitting.
    
    \item We introduce a curated dataset of 296 small-molecule PD-L1 inhibitors with experimentally determined IC\textsubscript{50} values, providing a valuable resource to accelerate research for this critical immuno-oncology task for future model development in the low-data regime.
    
\end{itemize}

\subsection*{Related work}

Advancements in AI/ML, particularly in deep generative modeling strategies, have led to a paradigm shift in de novo drug design \cite{mak2024artificial}. Various deep learning architectures such as recurrent neural networks (RNN) \cite{segler2018generating,bjerrum2017molecular}, reinforcement learning (RL) \cite{olivecrona2017molecular}, VAE \cite{kingma2013auto,gomez2018automatic}, GAN \cite{goodfellow2014generative,rathod2023unlocking}, and others have already been applied to generate novel molecules with desired properties. Both the text-based molecular representation (SMILES) and graph-based representation have been explored in deep generative model-based de novo drug design \cite{deng2022artificial}.

Conditional VAE (CVAEs) enables researchers to directly sample molecules with desired drug-like properties like MW, LogP, HBD, HBA, and TPSA \cite{lim2018molecular,kang2018conditional}. Junction tree variational autoencoder (JT-VAE) generates molecular graphs in two phases, by first generating a tree-structured scaffold over chemical sub-structures, and then combining them into a molecule with a graph message passing network \cite{jin2018junction}. Liao and co-workers introduced Sc2Mol framework, a two-step design process combining VAEs for scaffold generation and transformers for decoration, eliminating the need for predefined motifs \cite{liao2023sc2mol}. More recently, Liu et al. reported SmilesGEN, a dual-channel VAE architecture designed to incorporate drug-induced perturbations and transcriptional responses, thereby enabling the generation of molecules with biologically informed profiles \cite{liu2025phenotypic}.

In parallel, GAN-based models such as RANC (Reinforced Adversarial Neural Computer)  combines GAN and RL employed memory-augmented neural networks to produce drug-like and structurally unique compounds \cite{putin2018reinforced}.  Similarly, generative tensorial reinforcement learning (GENTRL) was used to find an efficient DDR1 kinase inhibitor \cite{zhavoronkov2019deep}. This transformer-based generative model, fine-tuned through reinforcement learning and transfer learning, has demonstrated notable improvements in molecular design, achieving higher synthetic feasibility, structural validity, and target-specific relevance. Recently, an RL-based architecture, RuSH (Reinforcement Learning for Unconstrained Scaffold Hopping), has been utilized for designing novel drug scaffolds. RuSH guides molecular generation toward creating compounds that maintain high three-dimensional and pharmacophore similarity to a reference molecule, while reducing scaffold similarity \cite{rossen2024scaffold}. 

Using a conditional recurrent neural network (cRNN) novel molecules were generated with desired properties. Selected molecular descriptors are embedded into the initial memory state of the network to guide generation. The architecture relies solely on an RNN-based decoder, making it conceptually simpler than a full conditional autoencoder \cite{kotsias2020direct}. Transformer architectures soon outpaced RNNs, with encoder-decoder models achieving 98.2\% chemical validity and greater diversity \cite{yang2021transformer}. REINVENT 4 consolidated these advances into a modern framework combining transformers, RNNs, curriculum learning, and optimization strategies for drug discovery \cite{loeffler2024reinvent}.

Recent advances in large language models (LLMs) have significantly influenced molecular generation strategies by applying natural language processing techniques to chemical representations \cite{chakraborty2023artificial,lee2025rag,liu2024drugagent}. MolGPT follows this approach as a transformer-decoder model trained on SMILES strings using a next-token prediction task, allowing it to generate drug-like molecules in an autoregressive manner \cite{bagal2021molgpt}. The model supports conditional generation, allowing control of desired molecular scaffolds and physicochemical properties (e.g., molecular weight, logP) by conditioning on scaffold SMILES and property vectors.

Simultaneously, contrastive learning has emerged as a powerful technique in drug discovery and cheminformatics \cite{gao2025revealing}. For instance, FragNet integrates contrastive objectives with transformer models to enhance the interpretability of molecular latent spaces \cite{shrivastava2021fragnet}. Similarly, CONSMI applied contrastive learning directly on SMILES, boosting generation quality through augmented molecular views \cite{qian2024consmi}.

Diffusion-based models like GeoDiff are also employed for molecular generation. The model treats each atom as a particle and learns to directly reverse the diffusion process (i.e., transforming from a noise distribution to stable conformations) as a Markov chain \cite{xu2022geodiff}. Most recently, structure-based models like CMD-GEN (Coarse-grained and Multi-dimensional Data-driven molecular generation) incorporated 3D protein-ligand data using pharmacophore-guided sampling and diffusion models. This hierarchical approach enabled the generation of conformationally stable bioactive molecules in scarce data settings \cite{zou2025structure}.
Complementary to molecule-centric approaches, recent work has introduced protein inverse problem formulations that treat structure recovery and design as generative modeling tasks directly on proteins using diffusion baseed methods \cite{levy2024solving, banerjee2025adampnp}. These efforts highlight the potential of inverse-problem-guided generative frameworks to unify structural biology and molecular design.

Despite the remarkable progress in generative modeling for de novo drug design, several intrinsic limitations remain across these methodologies. Most deep learning models require large datasets to train effectively (\cref{tab: large dataset table}). This includes VAEs, GANs, transformers, and diffusion models. In many real-world drug discovery scenarios, only a small number of active compounds are available. These low-data settings make it difficult to apply standard generative approaches. GANs often suffer from instability during training. They can collapse to a narrow range of outputs and fail to explore diverse chemical space. RL-based methods are sample-inefficient and highly sensitive to the design of the reward function. SMILES-based models can generate invalid molecules due to small changes in syntax. This makes them fragile and hard to control. Graph-based models provide more chemical structure awareness. But they are computationally expensive, especially for large molecules or hierarchical designs. Diffusion models, while powerful in generating high-quality and diverse molecules, require long training times and careful calibration of noise schedules. Their iterative denoising process is computationally intensive and makes them less suitable for rapid, iterative design cycles, especially in low-resource settings. These limitations underscore the need for more efficient, interpretable, and low-data-compatible generative frameworks that can balance chemical validity, diversity, and computational feasibility in real-world drug discovery.

\subsection*{This work}

In this study, we introduce a novel generative framework for the design of bioactive molecules targeting PD-L1 and kinase receptors. While established datasets of kinase inhibitors are publicly available and utilized for model development, no comprehensive database exists for PD-L1 inhibitors. To address this gap, we curated a novel dataset comprising reported PD-L1 inhibitors annotated with experimentally determined IC\textsubscript{50} values, enabling data-driven modeling of PD-L1-targeted chemical space. Our approach integrates a transformer-based encoder to extract rich molecular embeddings from SMILES representations, followed by a contrastive learning strategy that structures the latent space into chemically meaningful clusters. A Gaussian Mixture Model (GMM) is subsequently employed to sample from specific clusters associated with desirable pharmacological properties.  This unified pipeline enables efficient exploration and generation of target-specific, structurally novel molecules with optimized drug-like properties.To assess the quality of the generated molecules, we conducted a series of in silico evaluations, including molecular docking and molecular dynamics (MD) simulations. Furthermore, the performance of our framework was benchmarked against state-of-the-art generative models such as JT-VAE and MolGPT, using key metrics including docking scores, molecular uniqueness, novelty, and Tanimoto similarity. Collectively, this approach offers a scalable and data-efficient solution for target-driven molecule generation, particularly in low-data therapeutic domains.

\begin{figure}[!ht]
    \centering
    \includegraphics[width=1.0\linewidth]{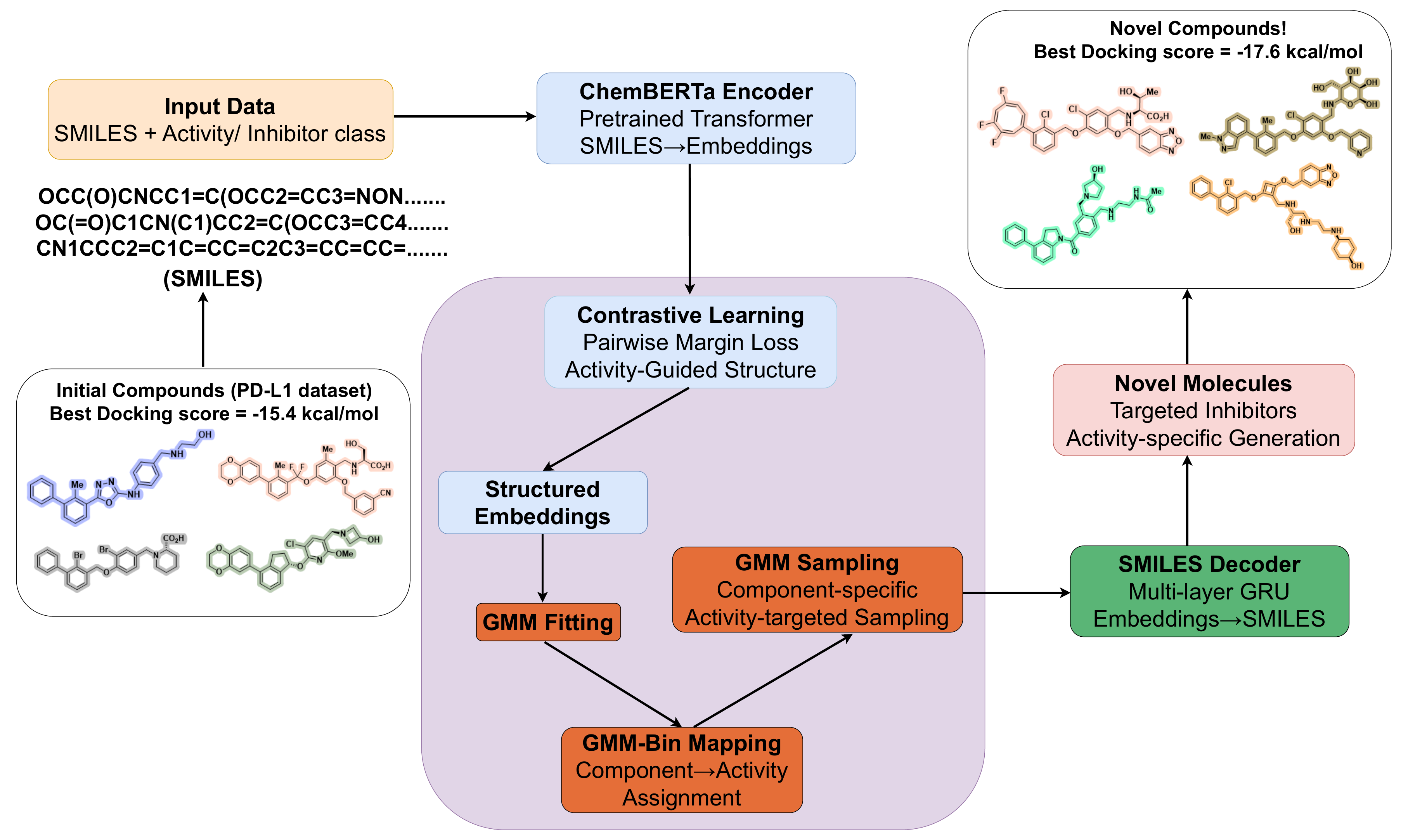}
    \caption{Workflow for GMM-based sampling and SMILES generation. Contrastive embeddings from ChemBERTa are clustered using a GMM. vectors are sampled per cluster, decoded into SMILES, filtered using RDKit, and saved for downstream applications.}
    \label{fig:gmm_workflow}
\end{figure}

\section*{Methodology}
The VECTOR+ framework is a pipeline designed to learn property-centric molecular representations and leverage them for targeted de novo generation. The methodology consists of three main stages: (1) structuring a latent space using property-guided contrastive representation learning; (2) modeling the learned chemical space with class-conditional probabilistic models; and (3) sampling from these models to generate novel molecules, with an optional reward-guided refinement step.

\subsection*{Property-Guided Contrastive Representation Learning}
Let a dataset be defined as $\mathcal{D} = \{(s_i, b_i)\}_{i=1}^n$, where $s_i \in \mathcal{S}$ is the SMILES representation of a molecule and $b_i \in \{1, \dots, C\}$ is its corresponding discrete property class label. Our primary goal is to learn an embedding function $M_\theta: \mathcal{S} \to \mathbb{R}^d$ that maps molecules into a $d$-dimensional latent space $\mathcal{Z}$ where geometric proximity correlates with property similarity.

The encoder $M_\theta$ is a composite model, $M_\theta = h_\theta \circ f_{\text{BERT}}$, where $f_{\text{BERT}}$ is a pre-trained ChemBERTa transformer and $h_\theta$ is a task-specific, non-linear projection head. For each molecule $s_i$, the encoder produces a latent vector $\mathbf{z}_i = M_\theta(s_i)$.

\paragraph{Architecture and Embedding Strategy.}
We used ChemBERTa, a transformer model trained on large sets of SMILES data, to encode each molecule \cite{chithrananda2020chemberta}. First, the SMILES strings were tokenized and passed through ChemBERTa. Then, we extracted the output from the [CLS] token to get a 768-dimensional vector that represents the molecule.

To prepare the embeddings for contrastive learning, we added a projection head. It had two fully connected layers with ELU activation. This head converted the raw embeddings into a new space where similar molecules were placed closer together, and different ones were pushed apart. These refined embeddings were then used for clustering and molecule generation.

\paragraph{Datasets and Binning Strategy:}
To use contrastive learning, the PD-L1 dataset was grouped into discrete bins based on activity values, and the kinase dataset was divided into four classes according to their modes of action:

\begin{itemize}
    \item \textbf{PD-L1 Task:} This dataset contains 296 molecules with experimentally measured IC\textsubscript{50} values. We applied a $\log (\mathrm{IC}_{50}$) transformation and then split the data at the median. Molecules with $\log (\mathrm{IC}_{50}$) values less than or equal to the median were labeled as ``high activity'', and the rest as ``low activity''.
    
    \item \textbf{Kinase Task:} This dataset includes kinase inhibitors categorized into four known activity classes. These labels were directly used as input for contrastive learning.
\end{itemize}

Even though the tasks differ, the same contrastive learning framework was applied to both datasets.

\paragraph{Contrastive Loss Function:}
The goal of the contrastive loss is to bring similar molecules closer in the embedding space and push dissimilar ones apart. 
Over the years, various contrastive loss functions such as InfoNCE, NT-Xent, and triplet loss have been explored for representation learning. In our case, we employed a pairwise contrastive loss to structure the latent space, ensuring that similar pairs are embedded closer together than dissimilar ones. For any pair of molecules $(s_i, s_j)$, we define a binary similarity label $y_{ij} = \mathbbm{1}[b_i = b_j]$, where $\mathbbm{1}[\cdot]$ is the indicator function. The objective is to minimize the distance between embeddings of molecules from the same class (positive pairs) while enforcing a separation margin $m > 0$ for molecules from different classes (negative pairs). The contrastive loss for a mini-batch $\mathcal{B}$ of size $B$ is formulated as:
\begin{equation}
\mathcal{L}_{\text{contrast}}(\theta; \mathcal{B}) = \frac{1}{|\{(i,j): i \neq j\}|} \sum_{i \neq j} \left( y_{ij} \|\hat{\mathbf{z}}_i - \hat{\mathbf{z}}_j\|_p^2 + (1 - y_{ij}) \left[\max(0, m - \|\hat{\mathbf{z}}_i - \hat{\mathbf{z}}_j\|_p)\right]^2 \right)
\label{eq:contrastive_loss}
\end{equation}
where we use the $L_p$-norm with $p=1$. The encoder $M_\theta$ is trained by minimizing this loss over the entire dataset.

Here, $m = 1.0$ is the margin that defines how far apart different-class embeddings should be. We computed the final loss by averaging over the upper triangle of the distance matrix to avoid counting the same pairs twice.

\subsection*{Latent Space Modeling and Targeted Generation}

\paragraph{Probabilistic Clustering and Cluster-to-Class Alignment:}
Upon convergence of the contrastive training, the latent space $\mathcal{Z}$ is organized into distinct, property-aligned regions. To formalize this structure, we model the entire embedding set $Z = \{\mathbf{z}_i\}_{i=1}^n$ with a GMM. We hypothesize that the data is generated from a mixture of $K$ Gaussian components, where we set $K=C$ to match the number of property classes. The probability density of a point $\mathbf{z}$ is given by:

\begin{equation}
p(\mathbf{z}|\Theta) = \sum_{k=1}^{K}\pi_k\mathcal{N}(\mathbf{z};\boldsymbol{\mu}_k,\boldsymbol{\Sigma}_k)
\label{eq:gmm_density}
\end{equation}

where the parameters $\Theta = \{\pi_k, \boldsymbol{\mu}_k, \boldsymbol{\Sigma}_k\}_{k=1}^K$ represent the mixture weights, means, and covariance matrices. These parameters are estimated by maximizing the data log-likelihood using the Expectation-Maximization (EM) algorithm.
The EM algorithm iteratively alternates between two steps until convergence:
\begin{itemize}
    \item \textbf{Expectation (E-step):} In this step, the expected cluster assignments for each data point given the current model parameters $\Theta$ is computed. This is done by calculating the posterior probability, or \textit{responsibility}, $r_{ik}$ that component $k$ is responsible for generating data point $\mathbf{z}_i$:
    \begin{equation}
        r_{ik} = \frac{\pi_k \mathcal{N}(\mathbf{z}_i | \boldsymbol{\mu}_k, \boldsymbol{\Sigma}_k)}{\sum_{j=1}^K \pi_j \mathcal{N}(\mathbf{z}_i | \boldsymbol{\mu}_j, \boldsymbol{\Sigma}_j)}
        \label{eq:responsibility}
    \end{equation}

    \item \textbf{Maximization (M-step):} In this step, the model parameters $\Theta$ is updated to maximize the log-likelihood, using the responsibilities computed in the E-step as soft weights. Let $N_k = \sum_{i=1}^n r_{ik}$ be the effective number of points assigned to component $k$. The parameters are updated as follows:
    \begin{align}
        \boldsymbol{\mu}_k^{\text{new}} &= \frac{1}{N_k} \sum_{i=1}^n r_{ik} \mathbf{z}_i \label{eq:mu_update} \\
        \boldsymbol{\Sigma}_k^{\text{new}} &= \frac{1}{N_k} \sum_{i=1}^n r_{ik} (\mathbf{z}_i - \boldsymbol{\mu}_k^{\text{new}})(\mathbf{z}_i - \boldsymbol{\mu}_k^{\text{new}})^\top \label{eq:sigma_update} \\
        \pi_k^{\text{new}} &= \frac{N_k}{n} \label{eq:pi_update}
    \end{align}
\end{itemize}

These two steps are repeated until the parameters converge. This clustering provides a set of $K$ probabilistic components. To utilize them for property-guided generation, we must establish a correspondence between these components and the $C$ ground-truth property labels. We achieve this by solving a linear assignment problem. First, we compute the posterior probability, or responsibility, that component $k$ is responsible for generating data point $\mathbf{z}_i$: $r_{ik} = p(k|\mathbf{z}_i, \Theta)$. We then construct a $K \times C$ affinity matrix $\Phi$, where each element $\Phi_{kc}$ quantifies the total responsibility of component $k$ for all data points belonging to the true class $c$:

\begin{equation}
\Phi_{kc} = \sum_{i: b_i=c} r_{ik}
\label{eq:affinity_matrix}
\end{equation}

We then seek the optimal permutation $\gamma^*: \{1, \dots, K\} \to \{1, \dots, C\}$ that maximizes the total affinity, solving the maximum weight bipartite matching problem:

\begin{equation}
\gamma^* = \arg\max_{\gamma\in\mathcal{P}_K}\sum_{k=1}^K\Phi_{k,\gamma(k)}
\label{eq:optimal_mapping}
\end{equation}
where $\mathcal{P}_K$ is the set of all permutations of $\{1, \dots, K\}$. This problem is solved efficiently using the Hungarian algorithm \cite{munkres1957algorithms}. The resulting mapping $\gamma^*$ provides an appropriate alignment between the unsupervised GMM components and the desired property classes.

\subsection*{Justification for the GMM-based Surrogate}
\label{sec:gmm_optimality}

\paragraph{Methodological Rationale:}
Our choice to model the latent space utilizes GMM (unsupervised approach), followed by a component-to-class matching step. While a simpler, supervised approach of fitting a Gaussian to each class's labeled data is possible, the GMM framework provides a complete probabilistic model of the entire latent space density, $p(z)$, capturing not only the class-conditional distributions but also their relative prevalence via mixture weights. 
Furthermore, its reliance on soft assignments in the EM algorithm provides robustness for ambiguous data points near cluster boundaries. This decouples the task of density estimation from class labeling, creating a more flexible framework. The unsupervised approach is favorable in the presence of noisy labels.
The component-to-class matching is a direct corollary of the preceding contrastive learning stage. The property-guided contrastive loss (\Cref{eq:contrastive_loss}) is specifically designed to engineer a latent space where embeddings from different classes form well-separated, unimodal clusters (as observed in Figure~\ref{fig:both}). The GMM is therefore applied to a space already structured for this task, making the alignment of its components with the true classes highly probable. The Hungarian algorithm \cite{munkres1957algorithms} is then simply the optimal method for solving the resulting one-to-one assignment problem.

\paragraph{Theoretical Justification:}
Here, we ground our use of a GMM as a surrogate for the true class-conditional latent distributions in a fundamental principle from information theory. Let $P_c$ be the unknown distribution of latent embeddings $Z=M_\theta(s)$ for a given class $c$. We show that, within the family of all Gaussian distributions, the one that matches the moments of $P_c$ is the unique minimizer of the forward Kullback-Leibler (KL) divergence. This makes our GMM-based sampling approach a natural choice.

\paragraph{Assumptions:}
For each class $c$, we assume the latent distribution $P_c$ has a finite mean $\boldsymbol{\mu}_c=\mathbb{E}_{P_c}[Z]$ and a finite, \emph{positive-definite} covariance $\boldsymbol{\Sigma}_c=\mathrm{Cov}_{P_c}(Z)\in\mathbb{S}_{++}^d$. All Gaussian covariances in the surrogate family are constrained to the SPD cone $\boldsymbol{\Lambda}\in\mathbb{S}_{++}^d$.

\begin{theorem}[I-projection onto the Gaussian family \& per-class surrogate optimality]
\label{thm:perclass_gaussian}
Let $P$ be any distribution on $\mathbb{R}^d$ with finite mean $\boldsymbol{\mu}$ and covariance $\boldsymbol{\Sigma} \in\mathbb{S}_{++}^d$. Among all Gaussians $q_{\boldsymbol{\nu},\boldsymbol{\Lambda}}=\mathcal{N}(\boldsymbol{\nu},\boldsymbol{\Lambda})$ with $\boldsymbol{\Lambda}\in\mathbb{S}_{++}^d$, the cross-entropy
$\mathrm{CE}(P\Vert q_{\boldsymbol{\nu},\boldsymbol{\Lambda}})\;=\;-\mathbb{E}_{X\sim P}\!\left[\log q_{\boldsymbol{\nu},\boldsymbol{\Lambda}}(X)\right]$
is uniquely minimized at $(\boldsymbol{\nu},\boldsymbol{\Lambda})=(\boldsymbol{\mu},\boldsymbol{\Sigma})$. Equivalently,
$\mathrm{CE}(P\Vert q)=H(P)+\mathrm{KL}(P\Vert q),$
so minimizing $\mathrm{CE}$ over Gaussians is the same as minimizing $\mathrm{KL}(P\Vert q)$ over Gaussians. In particular, for each class $c$, the Gaussian $Q_c^\star=\mathcal{N}(\boldsymbol{\mu}_c,\boldsymbol{\Sigma}_c)$ uniquely minimizes $\mathrm{CE}(P_c\Vert Q)$ over all Gaussian $Q$.
\end{theorem}

\begin{proof}
The Gaussian negative log-density is $-\log q_{\boldsymbol{\nu},\boldsymbol{\Lambda}}(x) = \tfrac12(x-\boldsymbol{\nu})^\top \boldsymbol{\Lambda}^{-1}(x-\boldsymbol{\nu}) + \tfrac12\log\!\det\boldsymbol{\Lambda} + \tfrac d2\log(2\pi)$. Taking the expectation under $X\sim P$ with mean $\boldsymbol{\mu}$ and covariance $\boldsymbol{\Sigma}$ yields:
\begin{align*}
\mathrm{CE}(P\Vert q_{\boldsymbol{\nu},\boldsymbol{\Lambda}})
&= \tfrac12\,\mathbb{E}\!\left[(X-\boldsymbol{\nu})^\top\boldsymbol{\Lambda}^{-1}(X-\boldsymbol{\nu})\right]
 + \tfrac12\log\!\det\boldsymbol{\Lambda} + \tfrac d2\log(2\pi) \\
&= \tfrac12\,\mathrm{tr}(\boldsymbol{\Lambda}^{-1}\boldsymbol{\Sigma})
 + \tfrac12(\boldsymbol{\mu}-\boldsymbol{\nu})^\top\boldsymbol{\Lambda}^{-1}(\boldsymbol{\mu}-\boldsymbol{\nu})
 + \tfrac12\log\!\det\boldsymbol{\Lambda} + \tfrac d2\log(2\pi)
\end{align*}
The optimization proceeds in two steps. First, for a fixed $\boldsymbol{\Lambda}\in\mathbb{S}_{++}^d$, the objective is a strictly convex quadratic in $\boldsymbol{\nu}$, and its gradient $\nabla_{\boldsymbol{\nu}}\,\mathrm{CE} = -\,\boldsymbol{\Lambda}^{-1}(\boldsymbol{\mu}-\boldsymbol{\nu})$ is zero only at the unique minimizer $\boldsymbol{\nu}=\boldsymbol{\mu}$. Substituting this solution back yields the reduced objective:
\[
f(\boldsymbol{\Lambda})=\tfrac12\,\mathrm{tr}(\boldsymbol{\Lambda}^{-1}\boldsymbol{\Sigma})+\tfrac12\,\log\!\det\boldsymbol{\Lambda}.
\]
Next, to minimize over the covariance, we re-parameterize in terms of the precision matrix $\boldsymbol{A}=\boldsymbol{\Lambda}^{-1}\in\mathbb{S}_{++}^d$, giving $g(\boldsymbol{A})\coloneqq f(\boldsymbol{A}^{-1}) =\tfrac12\,\mathrm{tr}(\boldsymbol{A}\boldsymbol{\Sigma})-\tfrac12\,\log\!\det\boldsymbol{A}$. Next, using $d(\log\det\boldsymbol{A})=\mathrm{tr}(\boldsymbol{A}^{-1}d\boldsymbol{A})$ and $d(\boldsymbol{A}^{-1})=-\boldsymbol{A}^{-1}(d\boldsymbol{A})\boldsymbol{A}^{-1}$, the gradient is:
\[
\nabla_{\boldsymbol{A}} g(\boldsymbol{A})=\tfrac12\,\boldsymbol{\Sigma}-\tfrac12\,\boldsymbol{A}^{-1}.
\]
Setting the gradient to zero gives the unique stationary point $\boldsymbol{A}^{-1}=\boldsymbol{\Sigma}$, i.e., $\boldsymbol{\Lambda}=\boldsymbol{\Sigma}$. Uniqueness of this minimum follows from the strict convexity of $g(\boldsymbol{A})$. The Hessian operator (in direction $\boldsymbol{H}$) $\nabla^2 g(\boldsymbol{A})[\boldsymbol{H}]=\tfrac12\,\boldsymbol{A}^{-1}\boldsymbol{H}\,\boldsymbol{A}^{-1}$, is positive definite, as its quadratic form $\tfrac12\,\mathrm{tr}(\boldsymbol{H}\boldsymbol{A}^{-1}\boldsymbol{H}\boldsymbol{A}^{-1}) =\tfrac12\,\|\boldsymbol{A}^{-1/2}\boldsymbol{H}\boldsymbol{A}^{-1/2}\|_F^2$ is strictly positive for any non-zero symmetric matrix $\boldsymbol{H}$. Thus $g$ (and therefore $f$) is strictly convex on $\mathbb{S}_{++}^d$, and $(\boldsymbol{\nu},\boldsymbol{\Lambda})=(\boldsymbol{\mu},\boldsymbol{\Sigma})$ is the unique minimizer. The per-class statement follows by setting $P=P_c$ and $(\boldsymbol{\mu},\boldsymbol{\Sigma})=(\boldsymbol{\mu}_c,\boldsymbol{\Sigma}_c)$.
\end{proof}

\paragraph{Implication for the VECTOR+ Pipeline.}
Our methodology first uses contrastive learning to shape the latent space $\mathcal{Z}$ into well-separated, class-aligned regions (as shown in Figure~\ref{fig:both}). The EM algorithm used to fit the GMM (Eq.~\ref{eq:gmm_density}) estimates the moments ($\boldsymbol{\mu}_k, \boldsymbol{\Sigma}_k$) for each component. Theorem~\ref{thm:perclass_gaussian} demonstrates that by modeling each class with the resulting Gaussian component, we are using an optimal surrogate for the true, unknown latent distribution $P_c$.
Importantly, the Theorem~\ref{thm:perclass_gaussian} justifies the use of per-class Gaussian surrogates but does not imply that EM-fitted mixture components will exactly coincide with the latent class-conditionals. In practice, since the number of Gaussian components and the number of classes are the same, after Hungarian alignment, the learned GMM components approximate these class distributions when contrastive separation is sufficiently strong. Therefore, sampling from the aligned GMM component $k^*$ to generate new molecules (Algorithm~\ref{alg:vector_generation}) is an intuitive approach for targeted exploration. Note that we do not assume latent Gaussianity. We adopt the mixtures of Gaussians as a tractable surrogate when the encoder is trained with \Cref{eq:contrastive_loss}.
This theoretical justification aligns with extensive evidence that contrastive learning methods implicitly perform a form of soft clustering, making the latent space amenable to mixture modeling  \cite{li2020prototypical_cl, Balestriero_cl_embedding, bansal2024understanding_cl}. We can observe that contrastive learning produces a representation where different classes behave as components of a mixture model. Based on Theorem~\ref{thm:perclass_gaussian}, GMM is a natural choice for modeling the structured latent space for generative tasks.

\subsection*{Generative Decoding and Controlled Molecular Sampling.}

A GRU-based decoder, $D_\phi: \mathbb{R}^d \to \mathcal{S}$ is trained independently to map latent vectors back into SMILES strings. Generation of novel molecules is performed by sampling from the aligned, class-conditional distributions. To generate a molecule of a target class $c^*$, we first identify the corresponding GMM component $k^* = (\gamma^*)^{-1}(c^*)$. A new latent vector is then sampled from this component, $\mathbf{z}_{\text{new}} \sim \mathcal{N}(\boldsymbol{\mu}_{k^*}, \boldsymbol{\Sigma}_{k^*})$, and decoded into a SMILES string $s_{\text{new}} = D_\phi(\mathbf{z}_{\text{new}})$. This process can be further refined using an optional property-guided hill-climbing procedure (e.g., utilizing Lipinski's Rule of Five as a reward function) to explore high-reward regions within the target cluster. These generated molecules showed both similarity to known actives and meaningful variation, providing new candidates for further screening. The full framework is summarized in \Cref{alg:vector_training} and \ref{alg:vector_generation}.

\begin{algorithm}[!ht]
\caption{VECTOR+ Training and Latent Structuring}
\label{alg:vector_training}
\begin{algorithmic}[1]
\Require Dataset $\mathcal{D}$, Encoder $M_\theta$, Decoder $D_\phi$, margin $m$, epochs $E$, GMM components $K=C$.
\Ensure Trained encoder $M_\theta$, decoder $D_\phi$, aligned GMM ($\Theta, \gamma^*$).

\State Train encoder $M_\theta$ using contrastive loss (Eq.~\ref{eq:contrastive_loss}) for $E$ epochs.
\State Train decoder $D_\phi$ to reconstruct SMILES from latent embeddings using cross-entropy loss.
\State Compute embeddings: $Z \leftarrow \{M_\theta(s_i)\}_{i=1}^n$.
\State Fit GMM to embeddings $Z$ (Eq.~\ref{eq:gmm_density}) to obtain parameters $\Theta = \{\pi_k, \boldsymbol{\mu}_k, \boldsymbol{\Sigma}_k\}_{k=1}^K$ using EM (Eq.~\ref{eq:responsibility}, \ref{eq:mu_update}, \ref{eq:sigma_update}, \ref{eq:pi_update}).
\State Compute affinity matrix $\Phi$ (Eq.~\ref{eq:affinity_matrix}); determine optimal cluster-to-class mapping $\gamma^*$ (Eq.~\ref{eq:optimal_mapping}) via Hungarian algorithm.
\State \Return $M_\theta, D_\phi, \Theta, \gamma^*$
\end{algorithmic}
\end{algorithm}

\begin{algorithm}[!ht]
\caption{VECTOR+ Molecular Generation with Optional Hill-Climbing}
\label{alg:vector_generation}
\begin{algorithmic}[1]
\Require
    Trained decoder $D_\phi$, GMM parameters $\Theta$, mapping $\gamma^*$;
    Target class $c^*$, number of samples $T$;
    Validity checker $v: \mathcal{S} \to \{0,1\}$;
    Optional reward function $R: \mathcal{S} \to \mathbb{R}$;
    Optional hill-climbing steps $H$, neighborhood size $k$, noise scale $\alpha$.
\Ensure Set $\mathcal{M}_{c^*}$ of generated molecules.

\State Identify GMM component: $k^* \leftarrow (\gamma^*)^{-1}(c^*)$.
\State Initialize generated set: $\mathcal{M}_{c^*} \leftarrow \emptyset$.
\For{$t=1,\dots,T$}
    \State Sample initial latent vector: $\mathbf{z}_{\text{current}}\sim\mathcal{N}(\boldsymbol{\mu}_{k^*}, \boldsymbol{\Sigma}_{k^*})$.
    \If{$R \not\equiv \emptyset$}
        \Comment{Refine with property-guided hill-climbing.}
        \For{$h=1,\dots,H$}
            \State Decode current SMILES: $s_{\text{current}}\leftarrow D_\phi(\mathbf{z}_{\text{current}})$.
            \If{not $v(s_{\text{current}})$} \textbf{break} \EndIf
            \State Estimate local variance: $\sigma^2 \leftarrow \frac{1}{k}\sum_{\mathbf{z}'\in\text{KNN}_k(\mathbf{z}_{\text{current}})}\|\mathbf{z}_{\text{current}}-\mathbf{z}'\|_2^2$.
            \State Sample perturbation: $\boldsymbol{\epsilon} \sim \mathcal{N}(0, \alpha^2\sigma^2 I)$.
            \State Propose new vector: $\mathbf{z}_{\text{new}} \leftarrow \mathbf{z}_{\text{current}} + \boldsymbol{\epsilon}$.
            \State Decode new SMILES: $s_{\text{new}} \leftarrow D_\phi(\mathbf{z}_{\text{new}})$.
            \If{$v(s_{\text{new}})$ \textbf{and} $R(s_{\text{new}}) > R(s_{\text{current}})$}
                \Comment{Accept the better molecule.}
                \State $\mathbf{z}_{\text{current}} \leftarrow \mathbf{z}_{\text{new}}$
            \EndIf
        \EndFor
    \EndIf
    \State $s_{\text{final}} \leftarrow D_\phi(\mathbf{z}_{\text{current}})$.
    \If{$v(s_{\text{final}})$ and $s_{\text{final}}$ is novel}
        \State Add $s_{\text{final}}$ to $\mathcal{M}_{c^*}$.
    \EndIf
\EndFor
\State \Return $\mathcal{M}_{c^*}$
\end{algorithmic}
\end{algorithm}

\subsection*{Docking and Molecular Dynamics Simulation Study}

A comprehensive in silico study was performed to identify potential hit molecules among the newly generated compounds. All computational studies were conducted using various modules of the Schrödinger software suite. For PD-L1, the crystal structure corresponding to PDB ID: 5J89 was utilized. In the case of kinase inhibitor studies, four representative protein structures were selected, each corresponding to a distinct class of kinase inhibitor: 6XM8 for Type I (I), 1XKK for Type I\nicefrac{1}{2} (I\nicefrac{1}{2}), 4ASD for Type II (II), and 3O96 for allosteric inhibitors (A).

All generated ligands were prepared using the LigPrep module under default settings, while protein structures were preprocessed and optimized using the Protein Preparation Wizard. Receptor grids were generated around the co-crystallized ligands using the Receptor Grid Generation module. Docking studies were carried out using the Glide docking program in extra precision (XP) mode, employing default parameters and the OPLS force field. The resulting docking poses were further visualized and analyzed using PyMOL.

To evaluate the dynamic stability of the protein–ligand complexes, molecular dynamics (MD) simulations were performed using Desmond. An orthorhombic simulation box was constructed around the complex, solvated using the TIP3P water model, and neutralized with appropriate counterions. The system was subjected to energy minimization, followed by a 250-nanosecond production MD run at 298 K and 1 atm pressure.

\section*{Results}

\subsection*{Data Collection and Data Preprocessing}

To initiate the discovery process for new PD-L1 inhibitors, we compiled a dataset of small molecules reported in the peer-reviewed literature \cite{wang2022discovery,yang2021design,guo2020design,ouyang2021design,song2021design,qin2019discovery,qin2021discovery}, each annotated with experimentally measured $\text{IC}_{50}$ values. After removing repeated entries and ensuring consistency, we finalized a curated set of 350 unique molecules with corresponding canonical SMILES representations and $\text{IC}_{50}$ values. The outliers were removed by using z-score method. After removing the outliers, a total of 296 molecules were selected for downstream modeling. 
The raw $\text{IC}_{50}$ distribution revealed a pronounced right-skew, with a large concentration of highly potent compounds ($\text{IC}_{50}$ < 100 nM) and a long tail extending to over 1000 nM (Fig.~\ref{fig: data preprocessing}C). This heavy-tailed, non-Gaussian distribution posed a challenge for the development of ML models, which typically assumes a more symmetric or normalized target distribution. To address this, we explored three standard transformation techniques, logarithmic, square root, and Box-Cox transformations. Among these, the Log transformation effectively stabilized variance and yielded a near-normal distribution suitable for learning-based tasks (Fig.~\ref{fig: data preprocessing}D). The log transformation was selected as the final preprocessing step for subsequent model development.

The kinase inhibitor dataset utilized in this study was taken from previously reported literature \cite{miljkovic2019machine}, which was originally sourced from the KLIFS database \cite{van2014klifs,kooistra2016klifs}. The dataset comprises four mechanistically distinct classes of inhibitors based on their binding modes to the kinase receptor: Type I (n = 1425), Type I\nicefrac{1}{2} (n = 394), Type II (n = 190), and allosteric inhibitors (n = 47) (Fig.~\ref{fig: data preprocessing}B). Type I inhibitors interact competitively with the ATP-binding site of kinases in their active conformation \cite{gavrin2013approaches}. In contrast, Type II inhibitors target the inactive conformation, occupying an induced hydrophobic pocket adjacent to the ATP site \cite{liu2006rational}. Type I\nicefrac{1}{2} inhibitors bind to an intermediate conformational state, exhibiting features of binding of Type I and Type II \cite{koeberle2012skepinone}. Lastly, allosteric inhibitors engage regions distinct from the ATP-binding pocket, modulating kinase activity through conformational changes at remote allosteric sites \cite{gavrin2013approaches}. Canonical SMILES strings were used to represent molecular structures and served as input for further machine learning analyses.

\subsection*{Contrastive Learning based on Data Labels}

To explore whether contrastive learning can enhance the separation of ligands based on biological activity, we started with pre-trained ChemBERTa embeddings and visualized them using UMAP. Initially, the embeddings showed no clear separation between the two reactivity bins corresponding to the PD-L1 dataset (Fig.~\ref{fig:both}A). However, after applying the contrastive learning framework, we observed a significant improvement in clustering (Fig.~\ref{fig:both}B). The model effectively pushed apart the embeddings of highly active and low active ligands, forming two well-separated clusters in the latent space. This shows that contrastive learning helped align the molecular representations with their biological activity, leading to more meaningful and interpretable clustering.

We also applied the same methodology on the kinase dataset that consisted of ligands classified into 4 different activity classes. In this case, each class was treated as a separate cluster. Initially the raw embedding did not form any distinct clusters (Fig.~\ref{fig:both}C) but interestingly, contrastive learning again proved effective, leading to clear and tight clustering into four distinct regions in the embedding space, each corresponding to a kinase activity class (Fig.~\ref{fig:both}D). This further supports the idea that contrastive learning can generalize well across datasets and capture activity-driven molecular patterns.

\begin{figure}[!htbp]
    \centering
    \begin{minipage}{0.9\textwidth}
        \centering
        \includegraphics[width=0.95\textwidth]{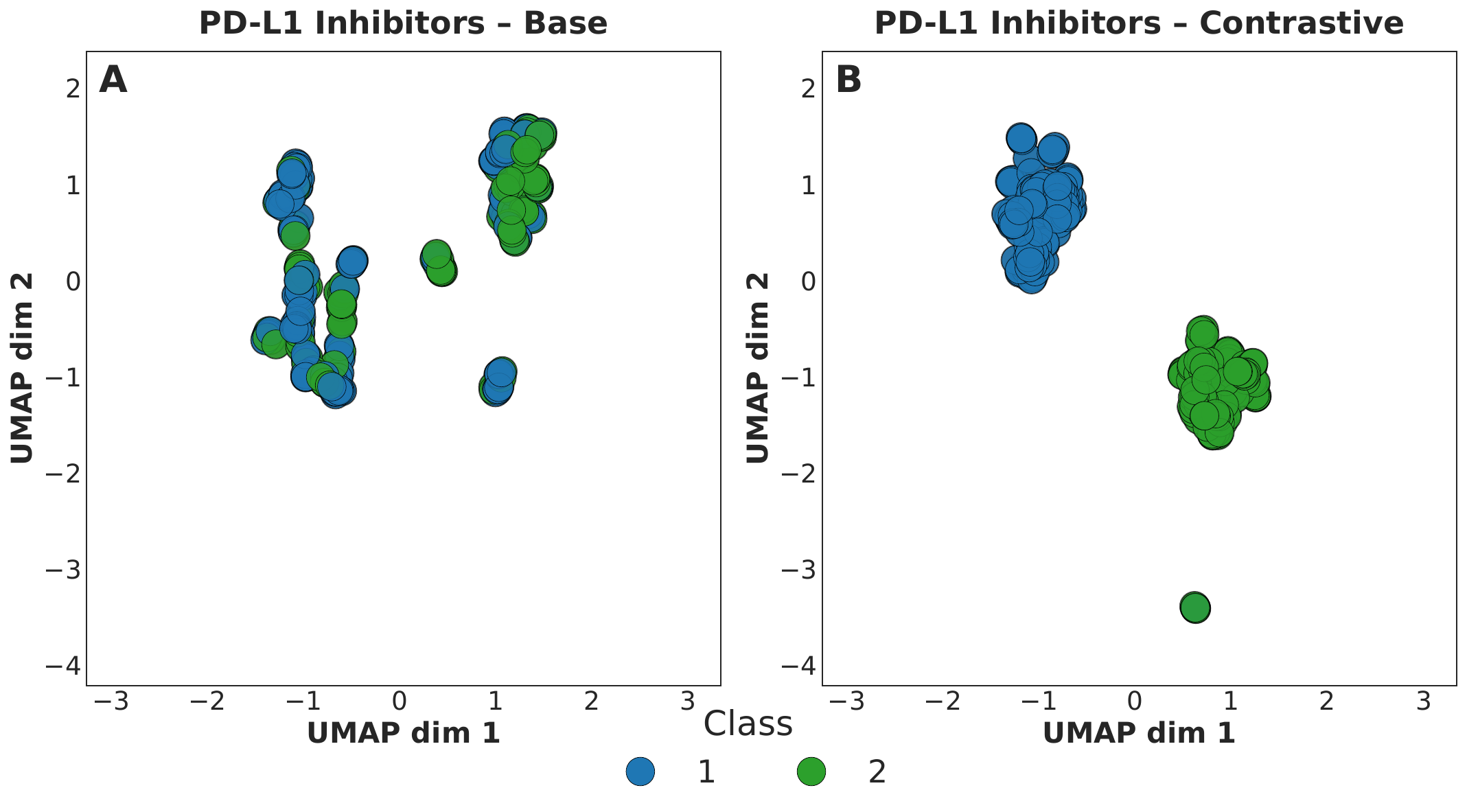}
        \label{fig:pdl1_base_vs_contrastive}
    \end{minipage}
    \begin{minipage}{0.9\textwidth}
        \centering
        \includegraphics[width=0.95\textwidth]{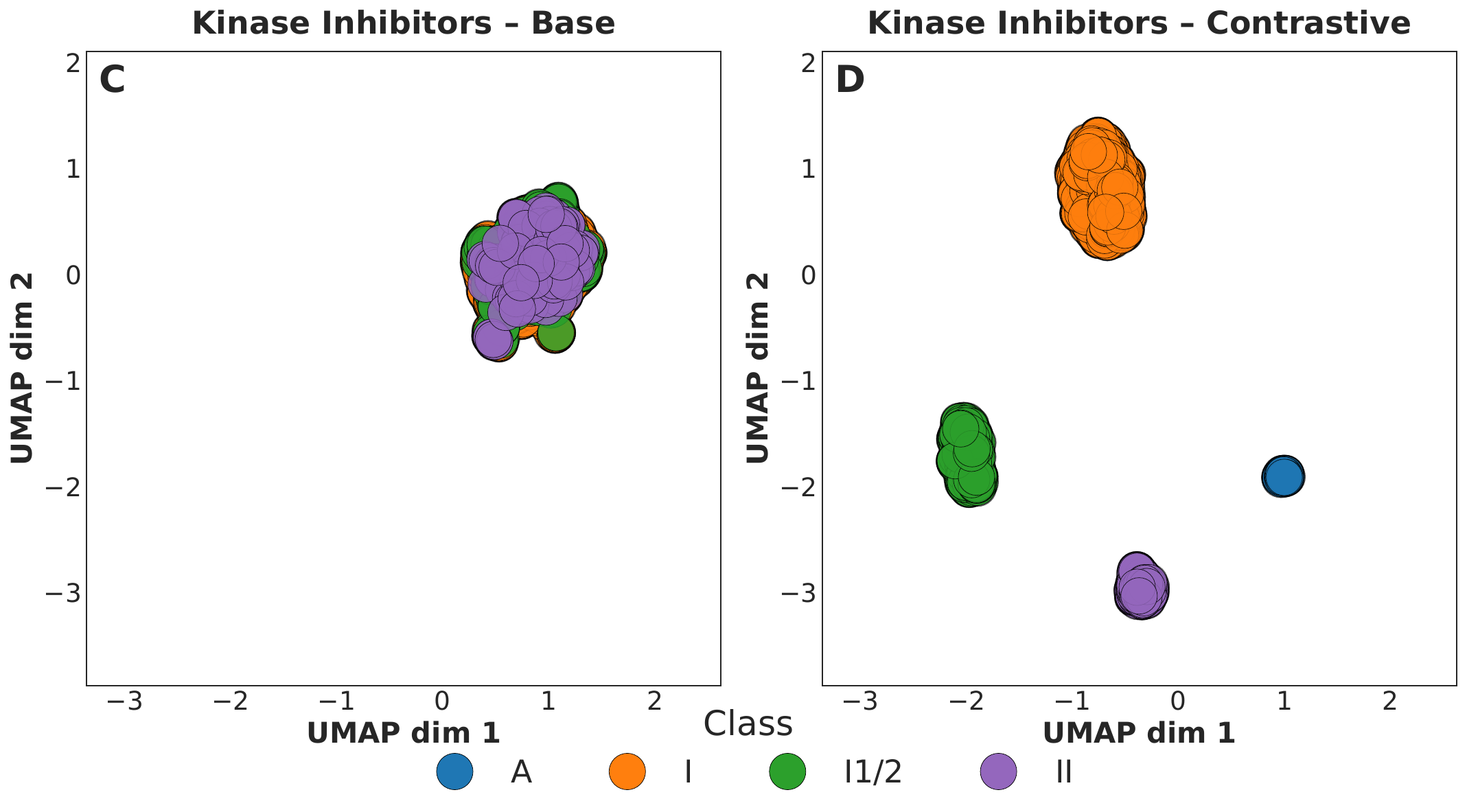}
        \label{fig:kinase_base_vs_contrastive}
    \end{minipage}
    \caption{This figure shows dimensionality reduction from 768-dimensional ChemBERTa embeddings to 2D UMAP space. (A-B) PD-L1 inhibitors classified into two activity groups based on log(IC\textsubscript{50}) values (Class 1: below median; Class 2: above median). In base embeddings (A), molecular representations show significant overlap between activity classes, while contrastive learning (B) successfully separates compounds into distinct clusters corresponding to their biological activity. (C-D) Kinase inhibitors categorized into four binding mode classes (A, I, I\nicefrac{1}{2}, and II). The base embeddings (C) show poor separation between the different binding modes, whereas the contrastive embeddings (D) produce clearly delineated clusters for each binding mode classification. These visualizations demonstrate how contrastive learning effectively restructures the embedding space to prioritize relevant features while maintaining molecular structural information, enabling more targeted molecule generation.} 
    \label{fig:both}
\end{figure}

\subsection*{GMM-Based Reactivity Clustering and Generation}

Once we obtained the 2-cluster contrastive embedding for PD-L1 (high vs. low activity), we trained a Gaussian Mixture Model (GMM) on the high-activity cluster to model the underlying distribution. The GMM was able to capture the molecular embedding density accurately. The sampled distribution closely matches the original high-activity data. We also visualized both densities together and saw strong overlap, confirming that GMM has effectively learned the latent space distribution.

From this learned distribution, we sampled 8,374 new molecular embeddings specifically from the high-activity cluster. These embeddings were then passed through our decoder to generate valid molecular structures. Representative molecules generated through this approach are shown below. Most of these resemble the chemical scaffolds found in our high-activity bin, but also explore novel chemical space around them (Fig.~\ref{fig:PDL1_kinase_compound_structure}).

We followed the same approach for the kinase dataset as well. After obtaining the 4-class contrastive embedding, we trained separate GMMs for each cluster. From these GMMs, we sampled 2500 molecular embeddings per class, resulting in a total of 10,000 kinase-like molecules. These generated compounds capture the density and diversity within each kinase class and will be valuable for downstream exploration and in-silico studies. 

This workflow demonstrates how the integration of contrastive learning and generative modeling enables the efficient production of chemically diverse molecules, even from very limited input data. Notably, from a set of just 47 Type A inhibitors, VECTOR+ was able to generate 2,500 novel, drug-like molecules, highlighting the beauty and efficiency of this pipeline for advancing both chemical and biological discovery.

\begin{figure}[!ht]
    \centering
    \includegraphics[width=1.0\linewidth]{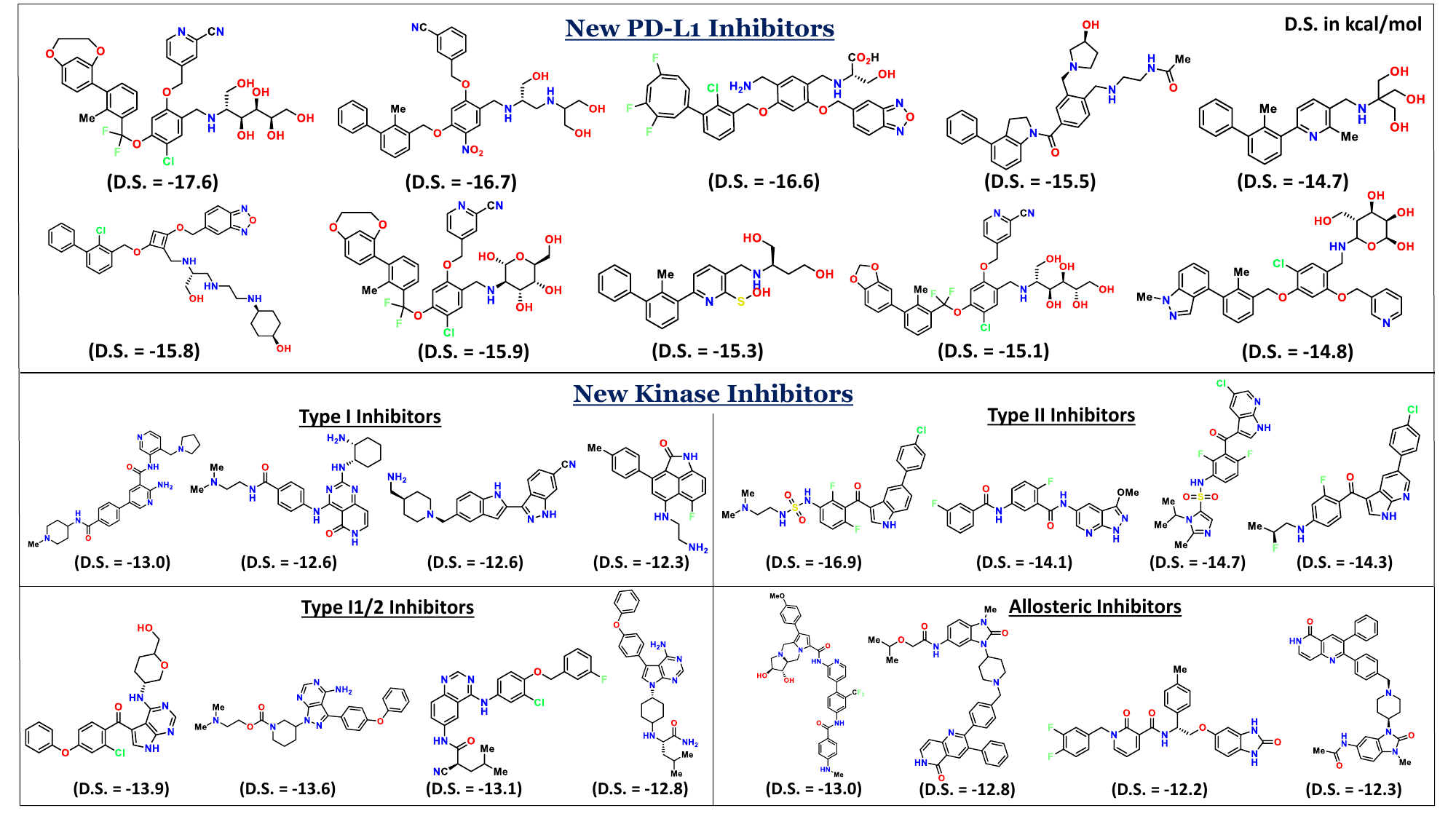}
    \caption{In silico studies were performed on the generated molecules from both the PD-L1 and kinase datasets. This figure displays the top-ranked candidate structures from each target along with their corresponding docking scores (in kcal/mol). The consistently high docking scores highlight the model’s ability to generate ligands with strong predicted binding affinity. Notably, structural analysis of the PD-L1 candidates revealed the emergence of an eight-membered ring motif, which was absent from the training set, demonstrating the model’s capacity to explore novel and unexplored regions of chemical space.}
    \label{fig:PDL1_kinase_compound_structure}
\end{figure}

\subsection*{Comparative Evaluation of VECTOR+ and Established Generative Models}

In this work, we chose JT-VAE and MolGPT as our main baselines. Both methods can be trained directly on small, SMILES-based datasets like ours, making them well-suited for a fair comparison. Other popular frameworks were not included because they require inputs or resources that do not match the scope of our study. For example, GeoDiff depends on 3D or DFT-level features, which would mean calculating quantum-chemical descriptors for every molecule in our PD-L1 and kinase sets, an impractical step for our current data. Similarly, CONSMI integrates a protein transformer as part of its training pipeline, but our datasets are ligand-only and do not provide paired protein information. REINVENT, while powerful, relies on custom reward functions (such as docking or ADMET scores) and large-scale sampling, which is difficult to calibrate in a low-data setting. By focusing on JT-VAE and MolGPT, we ensured that VECTOR+ was compared against models that are technically compatible and computationally feasible within our framework.

With this rationale in place, we next benchmarked the performance of VECTOR+ against JT-VAE and MolGPT using the PD-L1 dataset. The evaluation covered key metrics, including molecular weight, logP, docking score, structural uniqueness, chemical validity, and Tanimoto similarity.

Molecular weight and logP are fundamental physicochemical properties that influence drug absorption, distribution, and bioavailability. logP measures a compound’s lipophilicity, with values in an optimal range being desirable for drug-like behavior. VECTOR+ generated molecules with molecular weight and logP distributions closely matching those of the original PD-L1 compounds (Fig.~\ref{fig:model evaluation}A, ~\ref{fig:model evaluation}B). This suggests that VECTOR+ effectively captures the physicochemical profile of the target space, unlike JT-VAE and MolGPT, which showed greater deviations from the original distribution.

Chemical validity refers to the proportion of generated molecules that follow basic chemical rules, such as proper valency and atom connectivity. VECTOR+ and JT-VAE both achieved high validity, consistently generating chemically sound structures (Fig.~\ref{fig:model evaluation}D). MolGPT performed poorly in this regard, with validity scores below 5\%, indicating a failure to learn fundamental chemical syntax. This highlights a major weakness of large language model-based generators in low-data settings.

High uniqueness, which reflects the ability to generate a structurally diverse set of molecules, was observed in both VECTOR+ and MolGPT. JT-VAE, by contrast, exhibited low uniqueness, indicating limited exploration of chemical space (Fig.~\ref{fig:model evaluation}D).

Tanimoto similarity measures the structural overlap between two molecules based on their fingerprint representations, with higher scores indicating greater shared substructure \cite{bajusz2015tanimoto}. In this evaluation, VECTOR+ achieved the highest Tanimoto similarity to the original PD-L1 ligands. The Tanimoto similarity for our model was 0.65, which reflects its strong ability to generate molecules that stay close to the target chemical space while preserving sufficient chemical diversity. Both MolGPT and JT-VAE showed lower similarity scores, indicating weaker learning of relevant structural features (Fig.~\ref{fig:model evaluation}C, ~\ref{fig:model evaluation}D).

Synthetic accessibility (SA) score estimates how easily a molecule can be synthesized, with lower scores indicating higher feasibility.\cite{ertl2009estimation} VECTOR+ generated molecules with moderate SA scores, balancing structural complexity with practical synthesizability. MolGPT produced molecules with higher SA scores, suggesting that many of its outputs would be challenging to synthesize. JT-VAE generated simpler structures with lower SA scores, but at the cost of reduced structural diversity (Fig.~\ref{fig:model evaluation}E).

The Docking study was performed with the molecules generated by these models, and average docking scores were calculated for the top 10, 25, 50, 100, 150, and 200 compounds. In all cases, VECTOR+ consistently outperformed JT-VAE and MolGPT, demonstrating superior target-binding potential (\Cref{tab: Comparasion_table}). The top molecule generated by VECTOR+ achieved a docking score of –17.6 kcal/mol, in contrast to –12.4 kcal/mol for JT-VAE and -13.66 kcal/mol for MolGPT. Moreover, the average docking score of the top 50 molecules was markedly better for VECTOR+, JT-VAE, and MolGPT were –15.29 kcal/mol, –8.67 kcal/mol, and -11.04 kcal/mol, respectively, highlighting the ability of our model to generate molecules with superior binding affinity.

\begin{figure}[!ht]
    \centering
    \includegraphics[width=1.0\linewidth]{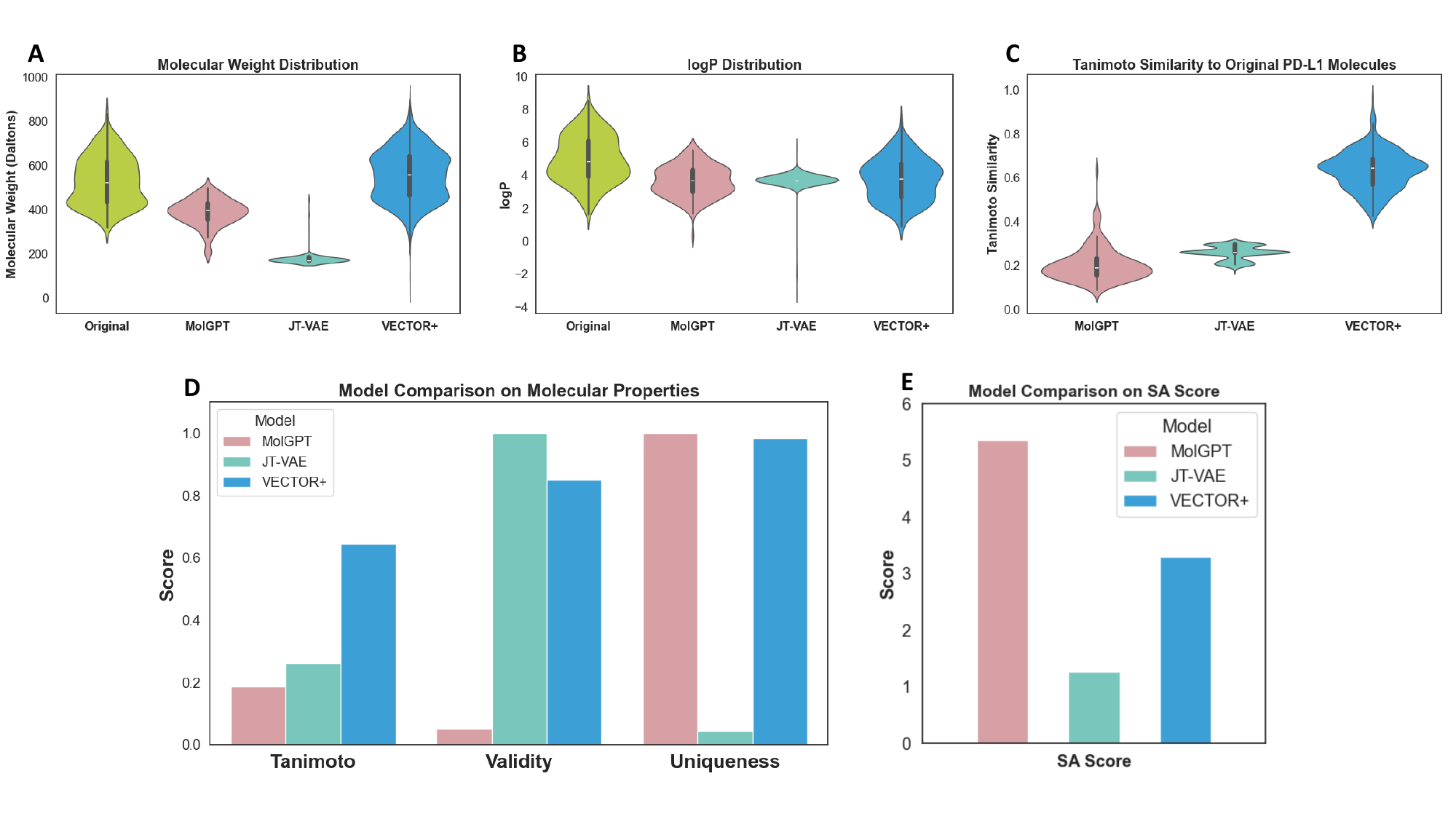}
    \caption{This figure shows the comparative evaluation of molecular properties and generative performance between the models.
(A) Molecular weight distributions of original PD-L1 ligands and molecules generated by MolGPT, JT-VAE, and VECTOR+, showing VECTOR+'s better alignment with the original distribution. (B) logP distribution comparison, where VECTOR+ captures the lipophilicity profile of the original molecules more accurately. (C) Tanimoto similarity to the original PD-L1 ligands, with VECTOR+ achieving the highest similarity, indicating strong structural alignment with the training space. (D) Bar plot comparing Tanimoto similarity, validity, and uniqueness across models, where VECTOR+ demonstrates superior structural accuracy and diversity. (E) Synthetic accessibility (SA) score comparison, where lower scores indicate higher synthetic feasibility; JT-VAE outperforms others in this metric, while VECTOR+ maintains a balance between accessibility and molecular realism.
}
    \label{fig:model evaluation}
\end{figure}

\begin{table}[H]
\centering
\caption{\textbf{Comparison of VECTOR+ with other generative models.}
In silico evaluation of molecules generated by VECTOR+ and benchmark generative models based on docking scores against the target protein PD-L1. The best docking scores are highlighted in \textbf{bold}.}
\label{tab: Comparasion_table}
\small
\begin{tabular}{lcccccc}
\toprule
\textbf{Molecules}          &&\textbf{Avg. Docking Score (kcal/mol) $\downarrow$}  \\
                  & VECTOR+                  & JT-VAE            & MolGPT \\
\midrule
Top molecule\hspace{25pt}      & \textbf{-17.6}            & -12.31           & -13.66          \\
Top 10            & \textbf{-16.16}           & -9.49            & -12.99          \\
Top 25            & \textbf{-15.63}           & -8.89            & -11.82           \\
Top 50            & \textbf{-15.12}           & -8.67            & -11.04           \\
Top 100           & \textbf{-14.58}           & -8.40            & -10.10          \\
Top 150           & \textbf{-14.20}           & -8.22            & -9.41          \\
Top 200           & \textbf{-13.84}           & -8.10            & -8.76            \\
\bottomrule
\end{tabular}
\end{table}

\subsection*{Computational Evaluation of GMM-Generated Ligands for PD-L1 and Kinase Inhibition}

Computational studies were performed to investigate the binding energies, interaction modes, and overall stability of the protein–ligand complexes. These analyses were carried out for both the newly designed molecules and previously reported compounds, and the results were subsequently compared. Molecular docking analyses offered detailed insights into the binding modes and interaction networks between the target protein and the ligands. Generally, stronger protein–ligand interactions correspond to more favorable (i.e., lower) docking scores, which are often associated with enhanced biological activity (i.e., lower IC\textsubscript{50} value).  However, it is important to note that no direct mathematical correlation exists between docking scores and biological outcomes.

\subsection*{PD-L1 Ligands}  

To prioritize promising candidates for downstream validation, extensive computational studies were conducted on a library of 8,374 structurally diverse molecules generated via GMM-based sampling from the latent space of the lowest $\text{IC}_{50}$ bin. Virtual screening was initially performed against the PD-L1 protein (PDB ID: 5J89) to evaluate binding potential. The top 50 hits exhibited docking scores ranging from -15.2 to -17.6 kcal/mol, notably outperforming the reference set of known PD-L1 inhibitors, which demonstrated docking affinities between -12.2 and -15.4 kcal/mol. This substantial increase in binding affinity highlights the superior interaction profiles of the newly generated compounds. Structural analysis revealed that the biphenyl moiety, a pharmacophoric feature consistently present in reported PD-L1 inhibitors, was retained in a significant proportion of the top-scoring molecules. Interestingly, several of the generated compounds also introduced novel structural motifs absent in the reference dataset. For example few of the generated molecules contain an eight-member ring, but it was absent in the training data. It indicates the model’s capacity to explore innovative chemical space while preserving key functional features. 

Protein–ligand interaction diagrams (Fig.~\ref{fig:PDL1_int_MD}A,  ~\ref{fig:PDL1_int_MD}B) revealed the formation of multiple stabilizing interactions. Specifically, key hydrogen bonds were observed with residues Thr-20, Asp-122, and Arg-125, alongside hydrophobic contacts with Tyr-56. These interactions are consistent with those observed in known PD-L1 binding modes, reinforcing the structural relevance of the generated compounds. Based on both structural novelty and synthetic accessibility, few compounds were selected for molecular dynamics (MD) simulations to evaluate their dynamic stability within the PD-L1 binding site. The stable trajectory in the MD simulations study over a 250 ns timescale, suggests a strong conformational persistence of the ligands within the active site (Fig.~\ref{fig:PDL1_int_MD}C,  ~\ref{fig:PDL1_int_MD}D). 

\begin{figure}[!ht]
    \centering
    \includegraphics[width=1.0\linewidth]{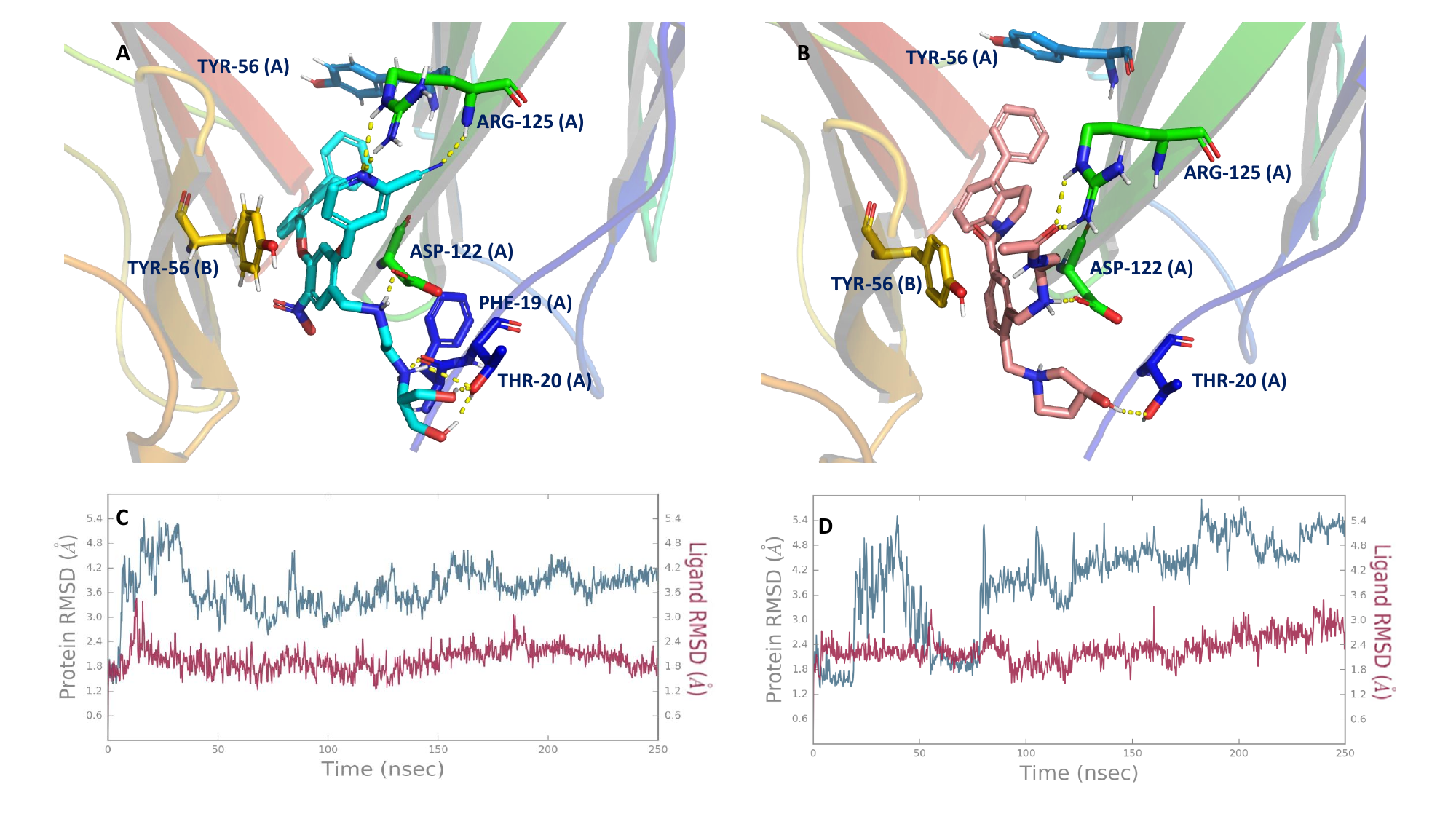}
    \caption{This figure showed molecular docking and MD simulation of top-ranked generated molecules. A total of 8,374 AI-generated molecules were virtually screened against PD-L1 (PDB ID: 5J89). (A) and (B) depict 2D interaction diagrams of two top-scoring candidates, illustrating multiple stabilizing interactions with key amino acid residues within the binding pocket. (C) and (D) present the 250 ns molecular dynamics simulation results for the corresponding complexes. The ligand RMSD remained below $2.5$ \r{A} throughout the simulations, indicating stable protein-ligand interactions and robust binding conformations.}
    \label{fig:PDL1_int_MD}
\end{figure}

\subsection*{Kinase Ligands}
A total of 2,500 molecules were generated for each inhibitor class and subjected to virtual screening to identify the top candidates per class (10,000 in total). Notably, for the allosteric class we trained on only 47 labeled inhibitors, yet the class-conditioned sampler produced 2,500 novel candidates (with the capacity to generate more), underscoring the framework’s effectiveness in a low-data regime. To benchmark the screening results, molecular docking studies were also performed using known kinase inhibitors such as Brigatinib, Lapatinib, Sorafenib, and IQO, as reference compounds for Type I, Type I\nicefrac{1}{2}, Type II, and allosteric inhibitors, respectively.

The docking analysis revealed that the top-ranked molecules in all four classes exhibited superior docking scores compared to their respective reference inhibitors. Notably, in the Type I and Type II classes, the lead compounds demonstrated significantly enhanced docking scores relative to Brigatinib and Sorafenib. For Type I, the top candidate achieved a docking score of –13.0 kcal/mol, surpassing Brigatinib's –11.0 kcal/mol. Similarly, the best-performing Type II inhibitor displayed a docking score of –16.89 kcal/mol, markedly better than Sorafenib's –13.1 kcal/mol. In both of these classes, the top ten molecules consistently outperformed the respective reference inhibitors. For the Type I\nicefrac{1}{2} and allosteric classes, the improvement was marginal but still favorable when compared to Lapatinib and IQO. Interaction analysis of the top-ranked ligands from the Type I and Type II classes revealed multiple stabilizing contacts with the kinase binding pocket, including key hydrogen bonds (Fig.~\ref{fig:Kinase_int_MD}A, ~\ref{fig:Kinase_int_MD}B). Furthermore, 250 ns molecular dynamics (MD) simulations confirmed the dynamic stability of both ligands at their respective binding sites, with minimal fluctuations observed over the simulation timescale (Fig.~\ref{fig:Kinase_int_MD}C,  ~\ref{fig:Kinase_int_MD}D).

\begin{figure}[!ht]
    \centering
    \includegraphics[width=1.0\linewidth]{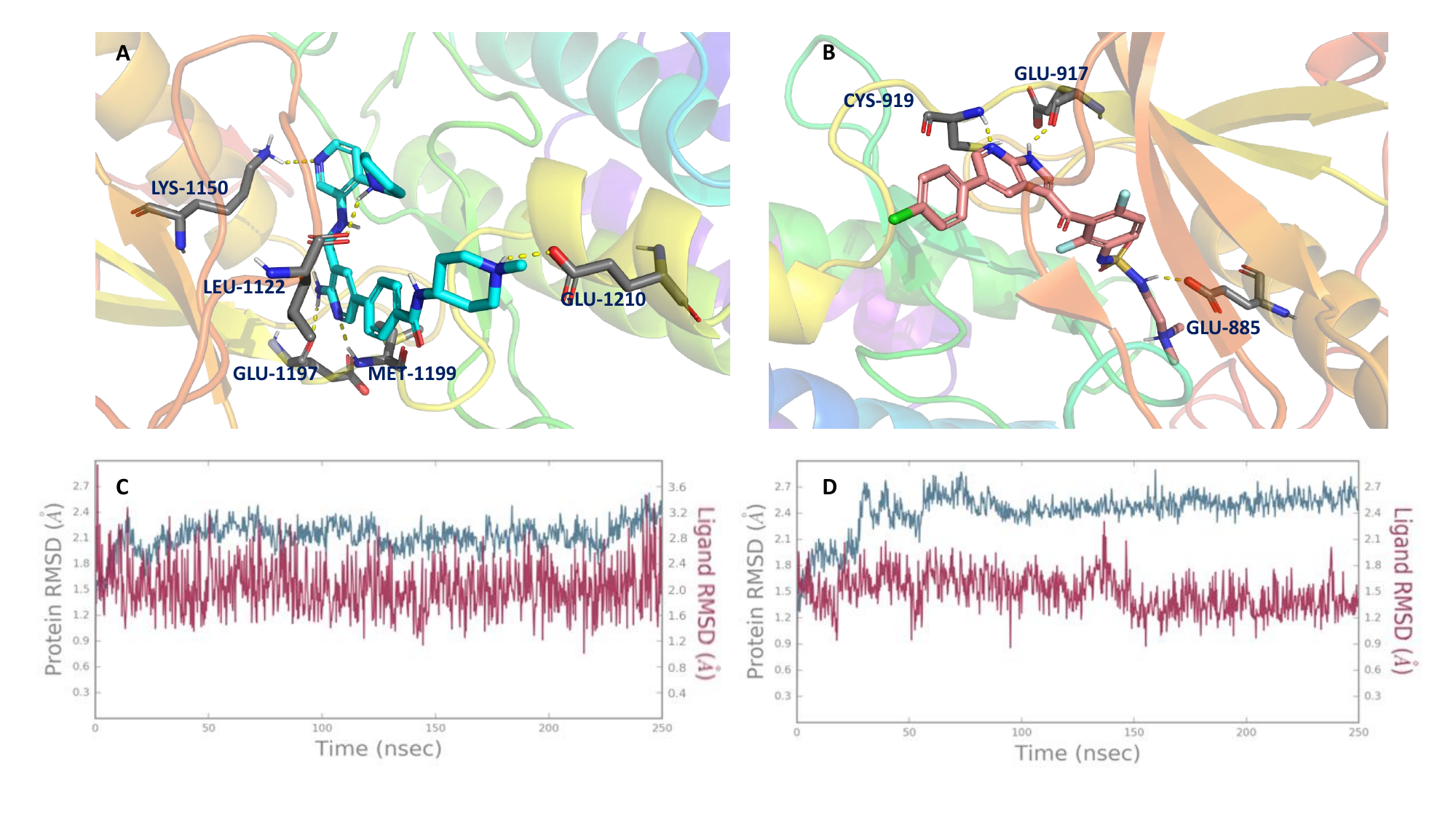}
    \caption{Molecular docking and MD simulations of generated kinase inhibitors. A total of 2,500 novel molecules were generated for each inhibitor class and subjected to virtual screening. (A) and (B) show 2D interaction diagrams of the top-ranked molecules from the Type I and Type II inhibitor classes, respectively, highlighting key interactions with residues in the kinase active site. (C) and (D) present 250 ns molecular dynamics simulation results for the corresponding complexes. The simulations revealed stable ligand trajectories, indicating robust and stable protein-ligand binding conformations.}
    \label{fig:Kinase_int_MD}
\end{figure}

Hence, the docking and MD simulation studies collectively support the potential of these GMM-generated ligands as promising PD-L1 and kinase inhibitors, warranting further experimental validation and the exploration of the structure-activity relationship (SAR).

\section*{Societal Impact}

Data scarcity remains a common problem in drug discovery across many therapeutic areas like oncology, antibiotic resistance, neglected tropical diseases, and rare genetic disorders and others. The present study shows that our generative model can effectively learn and generate potential drug molecules even with minimal data. Hence, our framework can be used to build molecular libraries tailored for a specific system with low data availability. This would speed up the discovery of new therapeutic leads in areas where traditional approaches fail.

Although the proposed framework has strong potential for advancing drug discovery, it is important to acknowledge the possible risks if misapplied. Generative models trained on chemical datasets could, in principle, be repurposed to design harmful substances, including highly toxic or environmentally hazardous compounds. Appropriate safeguards are therefore needed to ensure the technology is not exploited to generate toxic or dangerous chemicals.

\section*{Limitations and Future Work}

While our framework performs well in low-data settings,the quality of the generated molecules depends heavily on the SMILES decoder, which can sometimes produce repetitive, or hard-to-synthesize structures. Additionally, our current model relies exclusively on ligand-based information (SMILES) but does not incorporate protein features like binding site conformations or mutation-driven resistance mechanisms that influence biological activity. Future work will focus on integrating protein-level features to build a more comprehensive model and incorporating tools to check for synthetic feasibility and drug-likeness, which would generate candidates with a higher intrinsic probability of effective binding and reduce the reliance on downstream docking for screening. 

Beyond drug discovery, this methodology can be extended to other data-scarce domains like reaction development and catalysis to design novel ligands with improved efficiency. In summary, our framework is a broadly applicable tool for accelerating discovery in low-resource fields, including rare diseases, antibiotic resistance, and synthetic chemistry.

\section*{Conclusion}

The recent surge in machine learning for drug discovery has largely been driven by deep generative models trained on massive, general-purpose chemical datasets (ChEMBL, Zinc). However, a major challenge remains: these models often do not perform well when applied to specific problems with only small, carefully chosen datasets, a common scenario in real-world drug discovery. Current state-of-the-art architectures are often tailored for domain-specific applications by incorporating scaffold-based generation strategies or integrating protein–ligand interaction data during training. Yet, they remain fundamentally reliant on data volume rather than domain specificity.

To address this critical gap, we developed VECTOR+, a generative architecture specifically optimized for data-scarce, domain-focused applications. Unlike conventional approaches, VECTOR+ learns directly from compact datasets (often in the range of hundreds to a few thousand molecules), reflecting the real constraints of early-stage drug discovery. Our platform integrates a transformer-based encoder, contrastive learning-driven clustering, and Gaussian mixture model-based sampling within a unified pipeline. This synergy enables VECTOR+ to capture subtle, system-specific structure–activity relationships that broader models frequently overlook.

We demonstrated the capabilities of VECTOR+ in the context of two therapeutically relevant systems: PD-L1 and kinase receptors. Even in these low-data regimes, VECTOR+ was able to generate structurally novel and chemically diverse inhibitors, achieving higher docking scores compared to known actives and surpassing leading generative models across metrics such as novelty, uniqueness, and Tanimoto similarity. Ongoing synthesis and biological evaluation of these candidates further reinforce the practical impact of our approach.

By bridging the gap between advanced generative modeling and the realities of limited, domain-specific datasets, VECTOR+ establishes a new paradigm for efficient and targeted molecular discovery. While we emphasize our method's ability to operate effectively in data-scarce settings, the framework is not restricted to such scenarios: with more abundant labeled data, it can leverage richer supervision to further enhance clustering quality, sampling diversity, and property alignment. This strategy not only enables the design of drug-like molecules in challenging settings but also paves the way for broader applications, such as ligand optimization and reaction yield prediction, where data scarcity remains a limiting factor.

\subsection*{Acknowledgement}
The authors thank Harlin Lee (SDSS, UNC-CH) and Caroline Moosm\"{u}ller (Mathematics, UNC-CH) for their valuable feedback and comments.
The authors acknowledge Dhruv Goyal (Department of Mechanical Engineering, IIT Bombay) and Kashyap Khandelwal (Department of Computer Science and Engineering, IIT Bombay) for their valuable assistance in running JT-VAE baselines.

\bibliography{reference}

\begin{thebibliography}{10}
\urlstyle{rm}
\expandafter\ifx\csname url\endcsname\relax
  \def\url#1{\texttt{#1}}\fi
\expandafter\ifx\csname urlprefix\endcsname\relax\def\urlprefix{URL }\fi
\expandafter\ifx\csname doiprefix\endcsname\relax\def\doiprefix{DOI: }\fi
\providecommand{\bibinfo}[2]{#2}
\providecommand{\eprint}[2][]{\url{#2}}

\bibitem{wouters2020estimated}
\bibinfo{author}{Wouters, O.~J.}, \bibinfo{author}{McKee, M.} \& \bibinfo{author}{Luyten, J.}
\newblock \bibinfo{journal}{\bibinfo{title}{Estimated research and development investment needed to bring a new medicine to market, 2009-2018}}.
\newblock {\emph{\JournalTitle{Jama}}} \textbf{\bibinfo{volume}{323}}, \bibinfo{pages}{844--853} (\bibinfo{year}{2020}).

\bibitem{dara2022machine}
\bibinfo{author}{Dara, S.}, \bibinfo{author}{Dhamercherla, S.}, \bibinfo{author}{Jadav, S.~S.}, \bibinfo{author}{Babu, C.~M.} \& \bibinfo{author}{Ahsan, M.~J.}
\newblock \bibinfo{journal}{\bibinfo{title}{Machine learning in drug discovery: a review}}.
\newblock {\emph{\JournalTitle{Artificial intelligence review}}} \textbf{\bibinfo{volume}{55}}, \bibinfo{pages}{1947--1999} (\bibinfo{year}{2022}).

\bibitem{chen2018rise}
\bibinfo{author}{Chen, H.}, \bibinfo{author}{Engkvist, O.}, \bibinfo{author}{Wang, Y.}, \bibinfo{author}{Olivecrona, M.} \& \bibinfo{author}{Blaschke, T.}
\newblock \bibinfo{journal}{\bibinfo{title}{The rise of deep learning in drug discovery}}.
\newblock {\emph{\JournalTitle{Drug discovery today}}} \textbf{\bibinfo{volume}{23}}, \bibinfo{pages}{1241--1250} (\bibinfo{year}{2018}).

\bibitem{gupta2021artificial}
\bibinfo{author}{Gupta, R.} \emph{et~al.}
\newblock \bibinfo{journal}{\bibinfo{title}{Artificial intelligence to deep learning: machine intelligence approach for drug discovery}}.
\newblock {\emph{\JournalTitle{Molecular diversity}}} \textbf{\bibinfo{volume}{25}}, \bibinfo{pages}{1315--1360} (\bibinfo{year}{2021}).

\bibitem{biswas2020artificial}
\bibinfo{author}{Biswas, N.} \& \bibinfo{author}{Chakrabarti, S.}
\newblock \bibinfo{journal}{\bibinfo{title}{Artificial intelligence (ai)-based systems biology approaches in multi-omics data analysis of cancer}}.
\newblock {\emph{\JournalTitle{Frontiers in oncology}}} \textbf{\bibinfo{volume}{10}}, \bibinfo{pages}{588221} (\bibinfo{year}{2020}).

\bibitem{pitt2025real}
\bibinfo{author}{Pitt, W.~R.} \emph{et~al.}
\newblock \bibinfo{title}{Real-world applications and experiences of ai/ml deployment for drug discovery} (\bibinfo{year}{2025}).

\bibitem{yang2023application}
\bibinfo{author}{Yang, S.} \& \bibinfo{author}{Kar, S.}
\newblock \bibinfo{journal}{\bibinfo{title}{Application of artificial intelligence and machine learning in early detection of adverse drug reactions (adrs) and drug-induced toxicity}}.
\newblock {\emph{\JournalTitle{Artificial Intelligence Chemistry}}} \textbf{\bibinfo{volume}{1}}, \bibinfo{pages}{100011} (\bibinfo{year}{2023}).

\bibitem{sinha2023review}
\bibinfo{author}{Sinha, K.}, \bibinfo{author}{Ghosh, N.} \& \bibinfo{author}{Sil, P.~C.}
\newblock \bibinfo{journal}{\bibinfo{title}{A review on the recent applications of deep learning in predictive drug toxicological studies}}.
\newblock {\emph{\JournalTitle{Chemical Research in Toxicology}}} \textbf{\bibinfo{volume}{36}}, \bibinfo{pages}{1174--1205} (\bibinfo{year}{2023}).

\bibitem{tran2023artificial}
\bibinfo{author}{Tran, T. T.~V.}, \bibinfo{author}{Surya~Wibowo, A.}, \bibinfo{author}{Tayara, H.} \& \bibinfo{author}{Chong, K.~T.}
\newblock \bibinfo{journal}{\bibinfo{title}{Artificial intelligence in drug toxicity prediction: recent advances, challenges, and future perspectives}}.
\newblock {\emph{\JournalTitle{Journal of chemical information and modeling}}} \textbf{\bibinfo{volume}{63}}, \bibinfo{pages}{2628--2643} (\bibinfo{year}{2023}).

\bibitem{gangwal2024unlocking}
\bibinfo{author}{Gangwal, A.} \& \bibinfo{author}{Lavecchia, A.}
\newblock \bibinfo{journal}{\bibinfo{title}{Unlocking the potential of generative ai in drug discovery}}.
\newblock {\emph{\JournalTitle{Drug Discovery Today}}} \bibinfo{pages}{103992} (\bibinfo{year}{2024}).

\bibitem{kingma2013auto}
\bibinfo{author}{Kingma, D.~P.}, \bibinfo{author}{Welling, M.} \emph{et~al.}
\newblock \bibinfo{title}{Auto-encoding variational bayes} (\bibinfo{year}{2013}).

\bibitem{goodfellow2014generative}
\bibinfo{author}{Goodfellow, I.~J.} \emph{et~al.}
\newblock \bibinfo{journal}{\bibinfo{title}{Generative adversarial nets}}.
\newblock {\emph{\JournalTitle{Advances in neural information processing systems}}} \textbf{\bibinfo{volume}{27}} (\bibinfo{year}{2014}).

\bibitem{haroon2023generative}
\bibinfo{author}{Haroon, S.}, \bibinfo{author}{Hafsath, C.} \& \bibinfo{author}{Jereesh, A.}
\newblock \bibinfo{journal}{\bibinfo{title}{Generative pre-trained transformer (gpt) based model with relative attention for de novo drug design}}.
\newblock {\emph{\JournalTitle{Computational Biology and Chemistry}}} \textbf{\bibinfo{volume}{106}}, \bibinfo{pages}{107911} (\bibinfo{year}{2023}).

\bibitem{dou2023machine}
\bibinfo{author}{Dou, B.} \emph{et~al.}
\newblock \bibinfo{journal}{\bibinfo{title}{Machine learning methods for small data challenges in molecular science}}.
\newblock {\emph{\JournalTitle{Chemical Reviews}}} \textbf{\bibinfo{volume}{123}}, \bibinfo{pages}{8736--8780} (\bibinfo{year}{2023}).

\bibitem{parvatikar2023artificial}
\bibinfo{author}{Parvatikar, P.~P.} \emph{et~al.}
\newblock \bibinfo{journal}{\bibinfo{title}{Artificial intelligence: Machine learning approach for screening large database and drug discovery}}.
\newblock {\emph{\JournalTitle{Antiviral Research}}} \textbf{\bibinfo{volume}{220}}, \bibinfo{pages}{105740} (\bibinfo{year}{2023}).

\bibitem{tingle2023zinc}
\bibinfo{author}{Tingle, B.~I.} \emph{et~al.}
\newblock \bibinfo{journal}{\bibinfo{title}{Zinc-22- a free multi-billion-scale database of tangible compounds for ligand discovery}}.
\newblock {\emph{\JournalTitle{Journal of chemical information and modeling}}} \textbf{\bibinfo{volume}{63}}, \bibinfo{pages}{1166--1176} (\bibinfo{year}{2023}).

\bibitem{van2024deep}
\bibinfo{author}{van Tilborg, D.} \emph{et~al.}
\newblock \bibinfo{journal}{\bibinfo{title}{Deep learning for low-data drug discovery: Hurdles and opportunities}}.
\newblock {\emph{\JournalTitle{Current Opinion in Structural Biology}}} \textbf{\bibinfo{volume}{86}}, \bibinfo{pages}{102818} (\bibinfo{year}{2024}).

\bibitem{zaorsky2017causes}
\bibinfo{author}{Zaorsky, N.~G.} \emph{et~al.}
\newblock \bibinfo{journal}{\bibinfo{title}{Causes of death among cancer patients}}.
\newblock {\emph{\JournalTitle{Annals of oncology}}} \textbf{\bibinfo{volume}{28}}, \bibinfo{pages}{400--407} (\bibinfo{year}{2017}).

\bibitem{bray2021ever}
\bibinfo{author}{Bray, F.}, \bibinfo{author}{Laversanne, M.}, \bibinfo{author}{Weiderpass, E.} \& \bibinfo{author}{Soerjomataram, I.}
\newblock \bibinfo{journal}{\bibinfo{title}{The ever-increasing importance of cancer as a leading cause of premature death worldwide}}.
\newblock {\emph{\JournalTitle{Cancer}}} \textbf{\bibinfo{volume}{127}}, \bibinfo{pages}{3029--3030} (\bibinfo{year}{2021}).

\bibitem{zugazagoitia2016current}
\bibinfo{author}{Zugazagoitia, J.} \emph{et~al.}
\newblock \bibinfo{journal}{\bibinfo{title}{Current challenges in cancer treatment}}.
\newblock {\emph{\JournalTitle{Clinical therapeutics}}} \textbf{\bibinfo{volume}{38}}, \bibinfo{pages}{1551--1566} (\bibinfo{year}{2016}).

\bibitem{wang2018combining}
\bibinfo{author}{Wang, Y.} \emph{et~al.}
\newblock \bibinfo{journal}{\bibinfo{title}{Combining immunotherapy and radiotherapy for cancer treatment: current challenges and future directions}}.
\newblock {\emph{\JournalTitle{Frontiers in pharmacology}}} \textbf{\bibinfo{volume}{9}}, \bibinfo{pages}{185} (\bibinfo{year}{2018}).

\bibitem{schirrmacher2019chemotherapy}
\bibinfo{author}{Schirrmacher, V.}
\newblock \bibinfo{journal}{\bibinfo{title}{From chemotherapy to biological therapy: A review of novel concepts to reduce the side effects of systemic cancer treatment}}.
\newblock {\emph{\JournalTitle{International journal of oncology}}} \textbf{\bibinfo{volume}{54}}, \bibinfo{pages}{407--419} (\bibinfo{year}{2019}).

\bibitem{van2022chemotherapy}
\bibinfo{author}{van~den Boogaard, W.~M.}, \bibinfo{author}{Komninos, D.~S.} \& \bibinfo{author}{Vermeij, W.~P.}
\newblock \bibinfo{journal}{\bibinfo{title}{Chemotherapy side-effects: not all dna damage is equal}}.
\newblock {\emph{\JournalTitle{Cancers}}} \textbf{\bibinfo{volume}{14}}, \bibinfo{pages}{627} (\bibinfo{year}{2022}).

\bibitem{wang2016pd}
\bibinfo{author}{Wang, X.}, \bibinfo{author}{Teng, F.}, \bibinfo{author}{Kong, L.} \& \bibinfo{author}{Yu, J.}
\newblock \bibinfo{journal}{\bibinfo{title}{Pd-l1 expression in human cancers and its association with clinical outcomes}}.
\newblock {\emph{\JournalTitle{OncoTargets and therapy}}} \bibinfo{pages}{5023--5039} (\bibinfo{year}{2016}).

\bibitem{yu2020pd}
\bibinfo{author}{Yu, W.} \emph{et~al.}
\newblock \bibinfo{journal}{\bibinfo{title}{Pd-l1 promotes tumor growth and progression by activating wip and $\beta$-catenin signaling pathways and predicts poor prognosis in lung cancer}}.
\newblock {\emph{\JournalTitle{Cell death \& disease}}} \textbf{\bibinfo{volume}{11}}, \bibinfo{pages}{506} (\bibinfo{year}{2020}).

\bibitem{mahoney2015combination}
\bibinfo{author}{Mahoney, K.~M.}, \bibinfo{author}{Rennert, P.~D.} \& \bibinfo{author}{Freeman, G.~J.}
\newblock \bibinfo{journal}{\bibinfo{title}{Combination cancer immunotherapy and new immunomodulatory targets}}.
\newblock {\emph{\JournalTitle{Nature reviews Drug discovery}}} \textbf{\bibinfo{volume}{14}}, \bibinfo{pages}{561--584} (\bibinfo{year}{2015}).

\bibitem{yi2021regulation}
\bibinfo{author}{Yi, M.}, \bibinfo{author}{Niu, M.}, \bibinfo{author}{Xu, L.}, \bibinfo{author}{Luo, S.} \& \bibinfo{author}{Wu, K.}
\newblock \bibinfo{journal}{\bibinfo{title}{Regulation of pd-l1 expression in the tumor microenvironment}}.
\newblock {\emph{\JournalTitle{Journal of hematology \& oncology}}} \textbf{\bibinfo{volume}{14}}, \bibinfo{pages}{1--13} (\bibinfo{year}{2021}).

\bibitem{yamaoka2018receptor}
\bibinfo{author}{Yamaoka, T.}, \bibinfo{author}{Kusumoto, S.}, \bibinfo{author}{Ando, K.}, \bibinfo{author}{Ohba, M.} \& \bibinfo{author}{Ohmori, T.}
\newblock \bibinfo{journal}{\bibinfo{title}{Receptor tyrosine kinase-targeted cancer therapy}}.
\newblock {\emph{\JournalTitle{International journal of molecular sciences}}} \textbf{\bibinfo{volume}{19}}, \bibinfo{pages}{3491} (\bibinfo{year}{2018}).

\bibitem{tomuleasa2024therapeutic}
\bibinfo{author}{Tomuleasa, C.} \emph{et~al.}
\newblock \bibinfo{journal}{\bibinfo{title}{Therapeutic advances of targeting receptor tyrosine kinases in cancer}}.
\newblock {\emph{\JournalTitle{Signal Transduction and Targeted Therapy}}} \textbf{\bibinfo{volume}{9}}, \bibinfo{pages}{201} (\bibinfo{year}{2024}).

\bibitem{hsu2016role}
\bibinfo{author}{Hsu, J.~L.} \& \bibinfo{author}{Hung, M.-C.}
\newblock \bibinfo{journal}{\bibinfo{title}{The role of her2, egfr, and other receptor tyrosine kinases in breast cancer}}.
\newblock {\emph{\JournalTitle{Cancer and Metastasis Reviews}}} \textbf{\bibinfo{volume}{35}}, \bibinfo{pages}{575--588} (\bibinfo{year}{2016}).

\bibitem{lin2024regulatory}
\bibinfo{author}{Lin, X.} \emph{et~al.}
\newblock \bibinfo{journal}{\bibinfo{title}{Regulatory mechanisms of pd-1/pd-l1 in cancers}}.
\newblock {\emph{\JournalTitle{Molecular Cancer}}} \textbf{\bibinfo{volume}{23}}, \bibinfo{pages}{108} (\bibinfo{year}{2024}).

\bibitem{du2018mechanisms}
\bibinfo{author}{Du, Z.} \& \bibinfo{author}{Lovly, C.~M.}
\newblock \bibinfo{journal}{\bibinfo{title}{Mechanisms of receptor tyrosine kinase activation in cancer}}.
\newblock {\emph{\JournalTitle{Molecular cancer}}} \textbf{\bibinfo{volume}{17}}, \bibinfo{pages}{1--13} (\bibinfo{year}{2018}).

\bibitem{mak2024artificial}
\bibinfo{author}{Mak, K.-K.}, \bibinfo{author}{Wong, Y.-H.} \& \bibinfo{author}{Pichika, M.~R.}
\newblock \bibinfo{journal}{\bibinfo{title}{Artificial intelligence in drug discovery and development}}.
\newblock {\emph{\JournalTitle{Drug discovery and evaluation: safety and pharmacokinetic assays}}} \bibinfo{pages}{1461--1498} (\bibinfo{year}{2024}).

\bibitem{segler2018generating}
\bibinfo{author}{Segler, M.~H.}, \bibinfo{author}{Kogej, T.}, \bibinfo{author}{Tyrchan, C.} \& \bibinfo{author}{Waller, M.~P.}
\newblock \bibinfo{journal}{\bibinfo{title}{Generating focused molecule libraries for drug discovery with recurrent neural networks}}.
\newblock {\emph{\JournalTitle{ACS central science}}} \textbf{\bibinfo{volume}{4}}, \bibinfo{pages}{120--131} (\bibinfo{year}{2018}).

\bibitem{bjerrum2017molecular}
\bibinfo{author}{Bjerrum, E.~J.} \& \bibinfo{author}{Threlfall, R.}
\newblock \bibinfo{journal}{\bibinfo{title}{Molecular generation with recurrent neural networks (rnns)}}.
\newblock {\emph{\JournalTitle{arXiv preprint arXiv:1705.04612}}}  (\bibinfo{year}{2017}).

\bibitem{olivecrona2017molecular}
\bibinfo{author}{Olivecrona, M.}, \bibinfo{author}{Blaschke, T.}, \bibinfo{author}{Engkvist, O.} \& \bibinfo{author}{Chen, H.}
\newblock \bibinfo{journal}{\bibinfo{title}{Molecular de-novo design through deep reinforcement learning}}.
\newblock {\emph{\JournalTitle{Journal of cheminformatics}}} \textbf{\bibinfo{volume}{9}}, \bibinfo{pages}{1--14} (\bibinfo{year}{2017}).

\bibitem{gomez2018automatic}
\bibinfo{author}{G{\'o}mez-Bombarelli, R.} \emph{et~al.}
\newblock \bibinfo{journal}{\bibinfo{title}{Automatic chemical design using a data-driven continuous representation of molecules}}.
\newblock {\emph{\JournalTitle{ACS central science}}} \textbf{\bibinfo{volume}{4}}, \bibinfo{pages}{268--276} (\bibinfo{year}{2018}).

\bibitem{rathod2023unlocking}
\bibinfo{author}{Rathod, V.}, \bibinfo{author}{Gadilohar, J.}, \bibinfo{author}{Pawar, S.}, \bibinfo{author}{Joshi, A.} \& \bibinfo{author}{Sawant, S.}
\newblock \bibinfo{title}{Unlocking new possibilities in drug discovery: A gan-based approach}.
\newblock In \emph{\bibinfo{booktitle}{Artificial Intelligence-based Healthcare Systems}}, \bibinfo{pages}{135--144} (\bibinfo{publisher}{Springer}, \bibinfo{year}{2023}).

\bibitem{deng2022artificial}
\bibinfo{author}{Deng, J.}, \bibinfo{author}{Yang, Z.}, \bibinfo{author}{Ojima, I.}, \bibinfo{author}{Samaras, D.} \& \bibinfo{author}{Wang, F.}
\newblock \bibinfo{journal}{\bibinfo{title}{Artificial intelligence in drug discovery: applications and techniques}}.
\newblock {\emph{\JournalTitle{Briefings in Bioinformatics}}} \textbf{\bibinfo{volume}{23}} (\bibinfo{year}{2022}).

\bibitem{lim2018molecular}
\bibinfo{author}{Lim, J.}, \bibinfo{author}{Ryu, S.}, \bibinfo{author}{Kim, J.~W.} \& \bibinfo{author}{Kim, W.~Y.}
\newblock \bibinfo{journal}{\bibinfo{title}{Molecular generative model based on conditional variational autoencoder for de novo molecular design}}.
\newblock {\emph{\JournalTitle{Journal of cheminformatics}}} \textbf{\bibinfo{volume}{10}}, \bibinfo{pages}{1--9} (\bibinfo{year}{2018}).

\bibitem{kang2018conditional}
\bibinfo{author}{Kang, S.} \& \bibinfo{author}{Cho, K.}
\newblock \bibinfo{journal}{\bibinfo{title}{Conditional molecular design with deep generative models}}.
\newblock {\emph{\JournalTitle{Journal of chemical information and modeling}}} \textbf{\bibinfo{volume}{59}}, \bibinfo{pages}{43--52} (\bibinfo{year}{2018}).

\bibitem{jin2018junction}
\bibinfo{author}{Jin, W.}, \bibinfo{author}{Barzilay, R.} \& \bibinfo{author}{Jaakkola, T.}
\newblock \bibinfo{title}{Junction tree variational autoencoder for molecular graph generation}.
\newblock In \emph{\bibinfo{booktitle}{International conference on machine learning}}, \bibinfo{pages}{2323--2332} (\bibinfo{organization}{PMLR}, \bibinfo{year}{2018}).

\bibitem{liao2023sc2mol}
\bibinfo{author}{Liao, Z.}, \bibinfo{author}{Xie, L.}, \bibinfo{author}{Mamitsuka, H.} \& \bibinfo{author}{Zhu, S.}
\newblock \bibinfo{journal}{\bibinfo{title}{Sc2mol: a scaffold-based two-step molecule generator with variational autoencoder and transformer}}.
\newblock {\emph{\JournalTitle{Bioinformatics}}} \textbf{\bibinfo{volume}{39}}, \bibinfo{pages}{btac814} (\bibinfo{year}{2023}).

\bibitem{liu2025phenotypic}
\bibinfo{author}{Liu, H.}, \bibinfo{author}{Tian, S.} \& \bibinfo{author}{Liu, X.}
\newblock \bibinfo{journal}{\bibinfo{title}{Phenotypic profile-informed generation of drug-like molecules via dual-channel variational autoencoders}}.
\newblock {\emph{\JournalTitle{arXiv preprint arXiv:2506.02051}}}  (\bibinfo{year}{2025}).

\bibitem{putin2018reinforced}
\bibinfo{author}{Putin, E.} \emph{et~al.}
\newblock \bibinfo{journal}{\bibinfo{title}{Reinforced adversarial neural computer for de novo molecular design}}.
\newblock {\emph{\JournalTitle{Journal of chemical information and modeling}}} \textbf{\bibinfo{volume}{58}}, \bibinfo{pages}{1194--1204} (\bibinfo{year}{2018}).

\bibitem{zhavoronkov2019deep}
\bibinfo{author}{Zhavoronkov, A.} \emph{et~al.}
\newblock \bibinfo{journal}{\bibinfo{title}{Deep learning enables rapid identification of potent ddr1 kinase inhibitors}}.
\newblock {\emph{\JournalTitle{Nature biotechnology}}} \textbf{\bibinfo{volume}{37}}, \bibinfo{pages}{1038--1040} (\bibinfo{year}{2019}).

\bibitem{rossen2024scaffold}
\bibinfo{author}{Rossen, L.}, \bibinfo{author}{Sirockin, F.}, \bibinfo{author}{Schneider, N.} \& \bibinfo{author}{Grisoni, F.}
\newblock \bibinfo{journal}{\bibinfo{title}{Scaffold hopping with generative reinforcement learning}}.
\newblock {\emph{\JournalTitle{Journal of Chemical Information and Modeling}}}  (\bibinfo{year}{2024}).

\bibitem{kotsias2020direct}
\bibinfo{author}{Kotsias, P.-C.} \emph{et~al.}
\newblock \bibinfo{journal}{\bibinfo{title}{Direct steering of de novo molecular generation with descriptor conditional recurrent neural networks}}.
\newblock {\emph{\JournalTitle{Nature Machine Intelligence}}} \textbf{\bibinfo{volume}{2}}, \bibinfo{pages}{254--265} (\bibinfo{year}{2020}).

\bibitem{yang2021transformer}
\bibinfo{author}{Yang, L.} \emph{et~al.}
\newblock \bibinfo{journal}{\bibinfo{title}{Transformer-based generative model accelerating the development of novel braf inhibitors}}.
\newblock {\emph{\JournalTitle{ACS omega}}} \textbf{\bibinfo{volume}{6}}, \bibinfo{pages}{33864--33873} (\bibinfo{year}{2021}).

\bibitem{loeffler2024reinvent}
\bibinfo{author}{Loeffler, H.~H.} \emph{et~al.}
\newblock \bibinfo{journal}{\bibinfo{title}{Reinvent 4: Modern ai--driven generative molecule design}}.
\newblock {\emph{\JournalTitle{Journal of Cheminformatics}}} \textbf{\bibinfo{volume}{16}}, \bibinfo{pages}{20} (\bibinfo{year}{2024}).

\bibitem{chakraborty2023artificial}
\bibinfo{author}{Chakraborty, C.}, \bibinfo{author}{Bhattacharya, M.} \& \bibinfo{author}{Lee, S.-S.}
\newblock \bibinfo{journal}{\bibinfo{title}{Artificial intelligence enabled chatgpt and large language models in drug target discovery, drug discovery, and development}}.
\newblock {\emph{\JournalTitle{Molecular therapy Nucleic acids}}} \textbf{\bibinfo{volume}{33}}, \bibinfo{pages}{866--868} (\bibinfo{year}{2023}).

\bibitem{lee2025rag}
\bibinfo{author}{Lee, N.} \emph{et~al.}
\newblock \bibinfo{journal}{\bibinfo{title}{Rag-enhanced collaborative llm agents for drug discovery}}.
\newblock {\emph{\JournalTitle{arXiv preprint arXiv:2502.17506}}}  (\bibinfo{year}{2025}).

\bibitem{liu2024drugagent}
\bibinfo{author}{Liu, S.} \emph{et~al.}
\newblock \bibinfo{journal}{\bibinfo{title}{Drugagent: Automating ai-aided drug discovery programming through llm multi-agent collaboration}}.
\newblock {\emph{\JournalTitle{arXiv preprint arXiv:2411.15692}}}  (\bibinfo{year}{2024}).

\bibitem{bagal2021molgpt}
\bibinfo{author}{Bagal, V.}, \bibinfo{author}{Aggarwal, R.}, \bibinfo{author}{Vinod, P.} \& \bibinfo{author}{Priyakumar, U.~D.}
\newblock \bibinfo{journal}{\bibinfo{title}{Molgpt: molecular generation using a transformer-decoder model}}.
\newblock {\emph{\JournalTitle{Journal of chemical information and modeling}}} \textbf{\bibinfo{volume}{62}}, \bibinfo{pages}{2064--2076} (\bibinfo{year}{2021}).

\bibitem{gao2025revealing}
\bibinfo{author}{Gao, W.}, \bibinfo{author}{Raghavan, P.}, \bibinfo{author}{Shprints, R.} \& \bibinfo{author}{Coley, C.~W.}
\newblock \bibinfo{journal}{\bibinfo{title}{Revealing the relationship between publication bias and chemical reactivity with contrastive learning}}.
\newblock {\emph{\JournalTitle{Journal of the American Chemical Society}}} \textbf{\bibinfo{volume}{147}}, \bibinfo{pages}{8959--8968} (\bibinfo{year}{2025}).

\bibitem{shrivastava2021fragnet}
\bibinfo{author}{Shrivastava, A.~D.} \& \bibinfo{author}{Kell, D.~B.}
\newblock \bibinfo{journal}{\bibinfo{title}{Fragnet, a contrastive learning-based transformer model for clustering, interpreting, visualizing, and navigating chemical space}}.
\newblock {\emph{\JournalTitle{Molecules}}} \textbf{\bibinfo{volume}{26}}, \bibinfo{pages}{2065} (\bibinfo{year}{2021}).

\bibitem{qian2024consmi}
\bibinfo{author}{Qian, Y.}, \bibinfo{author}{Shi, M.} \& \bibinfo{author}{Zhang, Q.}
\newblock \bibinfo{journal}{\bibinfo{title}{Consmi: contrastive learning in the simplified molecular input line entry system helps generate better molecules}}.
\newblock {\emph{\JournalTitle{Molecules}}} \textbf{\bibinfo{volume}{29}}, \bibinfo{pages}{495} (\bibinfo{year}{2024}).

\bibitem{xu2022geodiff}
\bibinfo{author}{Xu, M.} \emph{et~al.}
\newblock \bibinfo{journal}{\bibinfo{title}{Geodiff: A geometric diffusion model for molecular conformation generation}}.
\newblock {\emph{\JournalTitle{arXiv preprint arXiv:2203.02923}}}  (\bibinfo{year}{2022}).

\bibitem{zou2025structure}
\bibinfo{author}{Zou, Y.} \emph{et~al.}
\newblock \bibinfo{journal}{\bibinfo{title}{A structure-based framework for selective inhibitor design and optimization}}.
\newblock {\emph{\JournalTitle{Communications Biology}}} \textbf{\bibinfo{volume}{8}}, \bibinfo{pages}{422} (\bibinfo{year}{2025}).

\bibitem{levy2024solving}
\bibinfo{author}{Levy, A.} \emph{et~al.}
\newblock \bibinfo{journal}{\bibinfo{title}{Solving inverse problems in protein space using diffusion-based priors}}.
\newblock {\emph{\JournalTitle{arXiv preprint arXiv:2406.04239}}}  (\bibinfo{year}{2024}).

\bibitem{banerjee2025adampnp}
\bibinfo{author}{Banerjee, A.}, \bibinfo{author}{Xu, X.}, \bibinfo{author}{Moosmüller, C.} \& \bibinfo{author}{Lee, H.}
\newblock \bibinfo{journal}{\bibinfo{title}{Adaptive multimodal protein plug-and-play with diffusion-based priors}}.
\newblock {\emph{\JournalTitle{arXiv preprint arXiv:2507.21260}}}  (\bibinfo{year}{2025}).

\bibitem{chithrananda2020chemberta}
\bibinfo{author}{Chithrananda, S.}, \bibinfo{author}{Grand, G.} \& \bibinfo{author}{Ramsundar, B.}
\newblock \bibinfo{journal}{\bibinfo{title}{Chemberta: large-scale self-supervised pretraining for molecular property prediction}}.
\newblock {\emph{\JournalTitle{arXiv preprint arXiv:2010.09885}}}  (\bibinfo{year}{2020}).

\bibitem{munkres1957algorithms}
\bibinfo{author}{Munkres, J.}
\newblock \bibinfo{journal}{\bibinfo{title}{Algorithms for the assignment and transportation problems}}.
\newblock {\emph{\JournalTitle{Journal of the society for industrial and applied mathematics}}} \textbf{\bibinfo{volume}{5}}, \bibinfo{pages}{32--38} (\bibinfo{year}{1957}).

\bibitem{li2020prototypical_cl}
\bibinfo{author}{Li, J.}, \bibinfo{author}{Zhou, P.}, \bibinfo{author}{Xiong, C.} \& \bibinfo{author}{Hoi, S.~C.}
\newblock \bibinfo{journal}{\bibinfo{title}{Prototypical contrastive learning of unsupervised representations}}.
\newblock {\emph{\JournalTitle{arXiv preprint arXiv:2005.04966}}}  (\bibinfo{year}{2020}).

\bibitem{Balestriero_cl_embedding}
\bibinfo{author}{Balestriero, R.} \& \bibinfo{author}{LeCun, Y.}
\newblock \bibinfo{title}{Contrastive and non-contrastive self-supervised learning recover global and local spectral embedding methods}.
\newblock In \bibinfo{editor}{Koyejo, S.} \emph{et~al.} (eds.) \emph{\bibinfo{booktitle}{Advances in Neural Information Processing Systems}}, vol.~\bibinfo{volume}{35}, \bibinfo{pages}{26671--26685} (\bibinfo{publisher}{Curran Associates, Inc.}, \bibinfo{year}{2022}).

\bibitem{bansal2024understanding_cl}
\bibinfo{author}{Bansal, P.}, \bibinfo{author}{Kavis, A.} \& \bibinfo{author}{Sanghavi, S.}
\newblock \bibinfo{journal}{\bibinfo{title}{Understanding self-supervised learning via gaussian mixture models}}.
\newblock {\emph{\JournalTitle{arXiv preprint arXiv:2411.03517}}}  (\bibinfo{year}{2024}).

\bibitem{wang2022discovery}
\bibinfo{author}{Wang, T.} \emph{et~al.}
\newblock \bibinfo{journal}{\bibinfo{title}{Discovery of small-molecule inhibitors of the pd-1/pd-l1 axis that promote pd-l1 internalization and degradation}}.
\newblock {\emph{\JournalTitle{Journal of medicinal chemistry}}} \textbf{\bibinfo{volume}{65}}, \bibinfo{pages}{3879--3893} (\bibinfo{year}{2022}).

\bibitem{yang2021design}
\bibinfo{author}{Yang, Y.}, \bibinfo{author}{Wang, K.}, \bibinfo{author}{Chen, H.} \& \bibinfo{author}{Feng, Z.}
\newblock \bibinfo{journal}{\bibinfo{title}{Design, synthesis, evaluation, and sar of 4-phenylindoline derivatives, a novel class of small-molecule inhibitors of the programmed cell death-1/programmed cell death-ligand 1 (pd-1/pd-l1) interaction}}.
\newblock {\emph{\JournalTitle{European Journal of Medicinal Chemistry}}} \textbf{\bibinfo{volume}{211}}, \bibinfo{pages}{113001} (\bibinfo{year}{2021}).

\bibitem{guo2020design}
\bibinfo{author}{Guo, J.} \emph{et~al.}
\newblock \bibinfo{journal}{\bibinfo{title}{Design, synthesis, and biological evaluation of linear aliphatic amine-linked triaryl derivatives as potent small-molecule inhibitors of the programmed cell death-1/programmed cell death-ligand 1 interaction with promising antitumor effects in vivo}}.
\newblock {\emph{\JournalTitle{Journal of Medicinal Chemistry}}} \textbf{\bibinfo{volume}{63}}, \bibinfo{pages}{13825--13850} (\bibinfo{year}{2020}).

\bibitem{ouyang2021design}
\bibinfo{author}{OuYang, Y.} \emph{et~al.}
\newblock \bibinfo{journal}{\bibinfo{title}{Design, synthesis, and evaluation of o-(biphenyl-3-ylmethoxy) nitrophenyl derivatives as pd-1/pd-l1 inhibitors with potent anticancer efficacy in vivo}}.
\newblock {\emph{\JournalTitle{Journal of Medicinal Chemistry}}} \textbf{\bibinfo{volume}{64}}, \bibinfo{pages}{7646--7666} (\bibinfo{year}{2021}).

\bibitem{song2021design}
\bibinfo{author}{Song, Z.} \emph{et~al.}
\newblock \bibinfo{journal}{\bibinfo{title}{Design, synthesis, and pharmacological evaluation of biaryl-containing pd-1/pd-l1 interaction inhibitors bearing a unique difluoromethyleneoxy linkage}}.
\newblock {\emph{\JournalTitle{Journal of Medicinal Chemistry}}} \textbf{\bibinfo{volume}{64}}, \bibinfo{pages}{16687--16702} (\bibinfo{year}{2021}).

\bibitem{qin2019discovery}
\bibinfo{author}{Qin, M.} \emph{et~al.}
\newblock \bibinfo{journal}{\bibinfo{title}{Discovery of [1, 2, 4] triazolo [4, 3-a] pyridines as potent inhibitors targeting the programmed cell death-1/programmed cell death-ligand 1 interaction}}.
\newblock {\emph{\JournalTitle{Journal of medicinal chemistry}}} \textbf{\bibinfo{volume}{62}}, \bibinfo{pages}{4703--4715} (\bibinfo{year}{2019}).

\bibitem{qin2021discovery}
\bibinfo{author}{Qin, M.} \emph{et~al.}
\newblock \bibinfo{journal}{\bibinfo{title}{Discovery of 4-arylindolines containing a thiazole moiety as potential antitumor agents inhibiting the programmed cell death-1/programmed cell death-ligand 1 interaction}}.
\newblock {\emph{\JournalTitle{Journal of medicinal chemistry}}} \textbf{\bibinfo{volume}{64}}, \bibinfo{pages}{5519--5534} (\bibinfo{year}{2021}).

\bibitem{miljkovic2019machine}
\bibinfo{author}{Miljkovic, F.}, \bibinfo{author}{Rodr{\'\i}guez-P{\'e}rez, R.} \& \bibinfo{author}{Bajorath, J.}
\newblock \bibinfo{journal}{\bibinfo{title}{Machine learning models for accurate prediction of kinase inhibitors with different binding modes}}.
\newblock {\emph{\JournalTitle{Journal of medicinal chemistry}}} \textbf{\bibinfo{volume}{63}}, \bibinfo{pages}{8738--8748} (\bibinfo{year}{2019}).

\bibitem{van2014klifs}
\bibinfo{author}{Van~Linden, O.~P.}, \bibinfo{author}{Kooistra, A.~J.}, \bibinfo{author}{Leurs, R.}, \bibinfo{author}{De~Esch, I.~J.} \& \bibinfo{author}{De~Graaf, C.}
\newblock \bibinfo{journal}{\bibinfo{title}{Klifs: a knowledge-based structural database to navigate kinase--ligand interaction space}}.
\newblock {\emph{\JournalTitle{Journal of medicinal chemistry}}} \textbf{\bibinfo{volume}{57}}, \bibinfo{pages}{249--277} (\bibinfo{year}{2014}).

\bibitem{kooistra2016klifs}
\bibinfo{author}{Kooistra, A.~J.} \emph{et~al.}
\newblock \bibinfo{journal}{\bibinfo{title}{Klifs: a structural kinase-ligand interaction database}}.
\newblock {\emph{\JournalTitle{Nucleic acids research}}} \textbf{\bibinfo{volume}{44}}, \bibinfo{pages}{D365--D371} (\bibinfo{year}{2016}).

\bibitem{gavrin2013approaches}
\bibinfo{author}{Gavrin, L.~K.} \& \bibinfo{author}{Saiah, E.}
\newblock \bibinfo{journal}{\bibinfo{title}{Approaches to discover non-atp site kinase inhibitors}}.
\newblock {\emph{\JournalTitle{MedChemComm}}} \textbf{\bibinfo{volume}{4}}, \bibinfo{pages}{41--51} (\bibinfo{year}{2013}).

\bibitem{liu2006rational}
\bibinfo{author}{Liu, Y.} \& \bibinfo{author}{Gray, N.~S.}
\newblock \bibinfo{journal}{\bibinfo{title}{Rational design of inhibitors that bind to inactive kinase conformations}}.
\newblock {\emph{\JournalTitle{Nature chemical biology}}} \textbf{\bibinfo{volume}{2}}, \bibinfo{pages}{358--364} (\bibinfo{year}{2006}).

\bibitem{koeberle2012skepinone}
\bibinfo{author}{Koeberle, S.~C.} \emph{et~al.}
\newblock \bibinfo{journal}{\bibinfo{title}{Skepinone-l is a selective p38 mitogen-activated protein kinase inhibitor}}.
\newblock {\emph{\JournalTitle{Nature chemical biology}}} \textbf{\bibinfo{volume}{8}}, \bibinfo{pages}{141--143} (\bibinfo{year}{2012}).

\bibitem{bajusz2015tanimoto}
\bibinfo{author}{Bajusz, D.}, \bibinfo{author}{R{\'a}cz, A.} \& \bibinfo{author}{H{\'e}berger, K.}
\newblock \bibinfo{journal}{\bibinfo{title}{Why is tanimoto index an appropriate choice for fingerprint-based similarity calculations?}}
\newblock {\emph{\JournalTitle{Journal of cheminformatics}}} \textbf{\bibinfo{volume}{7}}, \bibinfo{pages}{20} (\bibinfo{year}{2015}).

\bibitem{ertl2009estimation}
\bibinfo{author}{Ertl, P.} \& \bibinfo{author}{Schuffenhauer, A.}
\newblock \bibinfo{journal}{\bibinfo{title}{Estimation of synthetic accessibility score of drug-like molecules based on molecular complexity and fragment contributions}}.
\newblock {\emph{\JournalTitle{Journal of cheminformatics}}} \textbf{\bibinfo{volume}{1}}, \bibinfo{pages}{8} (\bibinfo{year}{2009}).

\bibitem{he2025generative}
\bibinfo{author}{He, Y.} \emph{et~al.}
\newblock \bibinfo{journal}{\bibinfo{title}{Generative molecule evolution using 3d pharmacophore for efficient structure-based drug design}}.
\newblock {\emph{\JournalTitle{arXiv preprint arXiv:2507.20130}}}  (\bibinfo{year}{2025}).

\bibitem{zhang2025zero}
\bibinfo{author}{Zhang, H.} \emph{et~al.}
\newblock \bibinfo{journal}{\bibinfo{title}{Zero-shot learning with subsequence reordering pretraining for compound-protein interaction}}.
\newblock {\emph{\JournalTitle{arXiv preprint arXiv:2507.20925}}}  (\bibinfo{year}{2025}).

\bibitem{gong2024text}
\bibinfo{author}{Gong, H.}, \bibinfo{author}{Liu, Q.}, \bibinfo{author}{Wu, S.} \& \bibinfo{author}{Wang, L.}
\newblock \bibinfo{title}{Text-guided molecule generation with diffusion language model}.
\newblock In \emph{\bibinfo{booktitle}{Proceedings of the AAAI Conference on Artificial Intelligence}}, vol.~\bibinfo{volume}{38}, \bibinfo{pages}{109--117} (\bibinfo{year}{2024}).

\bibitem{bou2024acegen}
\bibinfo{author}{Bou, A.} \emph{et~al.}
\newblock \bibinfo{journal}{\bibinfo{title}{Acegen: Reinforcement learning of generative chemical agents for drug discovery}}.
\newblock {\emph{\JournalTitle{Journal of Chemical Information and Modeling}}} \textbf{\bibinfo{volume}{64}}, \bibinfo{pages}{5900--5911} (\bibinfo{year}{2024}).

\bibitem{izdebski2025synergistic}
\bibinfo{author}{Izdebski, A.} \emph{et~al.}
\newblock \bibinfo{journal}{\bibinfo{title}{Synergistic benefits of joint molecule generation and property prediction}}.
\newblock {\emph{\JournalTitle{arXiv preprint arXiv:2504.16559}}}  (\bibinfo{year}{2025}).

\bibitem{banerjee2023surprisal}
\bibinfo{author}{Banerjee, A.} \emph{et~al.}
\newblock \bibinfo{journal}{\bibinfo{title}{Surprisal driven $ k $-nn for robust and interpretable nonparametric learning}}.
\newblock {\emph{\JournalTitle{arXiv preprint arXiv:2311.10246}}}  (\bibinfo{year}{2023}).

\end{thebibliography}

\newpage
\section*{Appendix}

\subsection*{Dataset Details}

For model development, the PD-L1 inhibitors were categorized into two groups based on their biological activity (Fig.~\ref{fig: data preprocessing}A). The raw $\text{IC}_{50}$ values exhibited a pronounced right-skewed distribution (Fig.~\ref{fig: data preprocessing}C), which was suboptimal for training models. To address this, the $\text{IC}_{50}$ values were log-transformed, resulting in an approximately normal distribution (Fig.~\ref{fig: data preprocessing}D) that was more suitable for downstream model development. On the other hand, the kinase inhibitors were classified into four distinct classes according to their binding modes. The four inhibitor class are Type I, Type II, Type I\nicefrac{1}{2} and Allosteric inhibitors (Fig.~\ref{fig: data preprocessing}B).

\begin{figure}[!ht]
    \centering
    \includegraphics[width=0.85\linewidth]{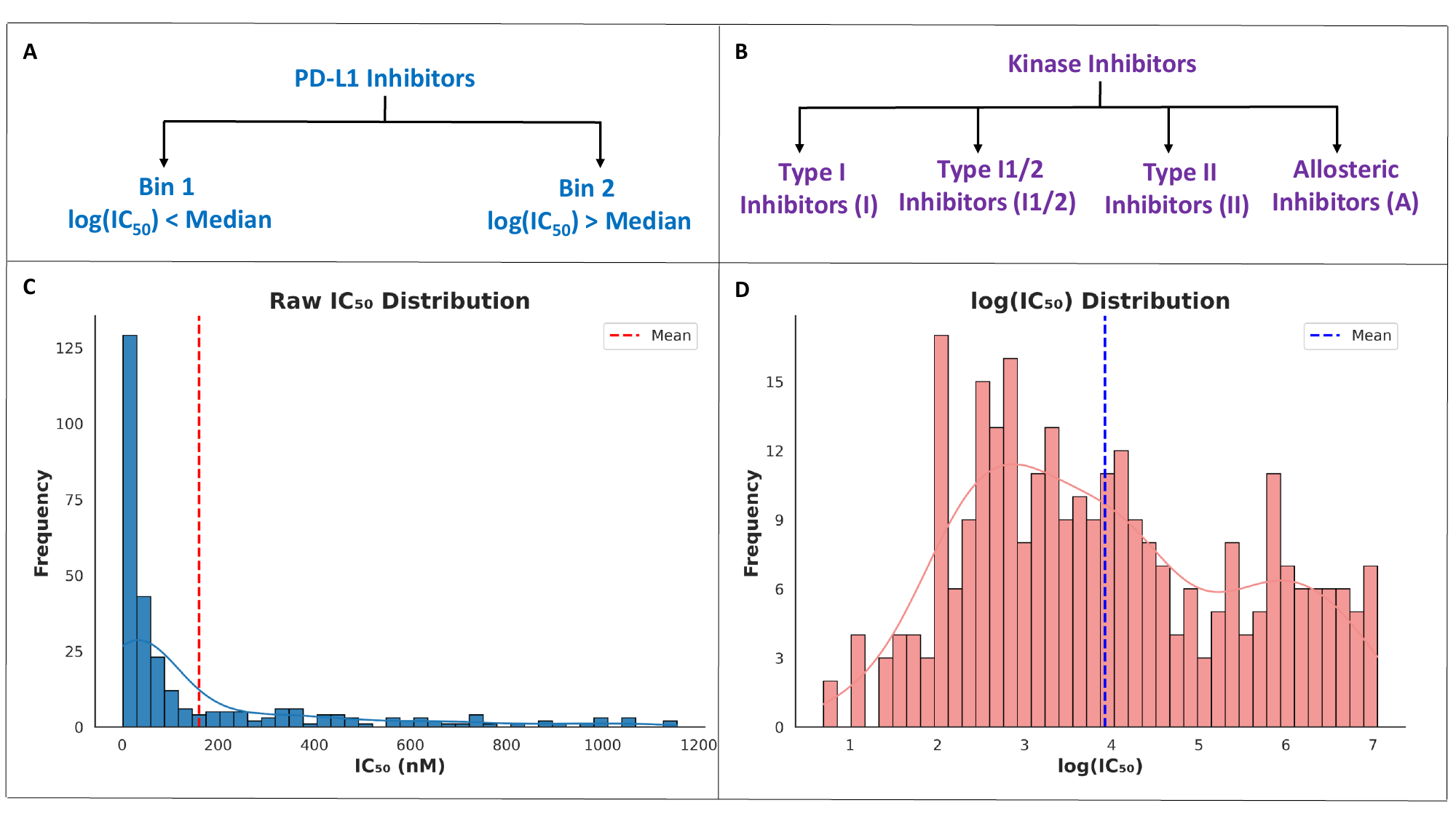}
    \caption{The figure showed the data segmentation and preprocessing applied prior to machine learning. (A) PD-L1 inhibitors dataset was categorized based on their biological activity. (B) Kinase inhibitor dataset was segmented according to the inhibitors’ mode of binding to the target protein. (C) The raw $\text{IC}_{50}$ values of the PD-L1 inhibitors exhibited a right-skewed distribution. (D) Log transformation of the $\text{IC}_{50}$ values of the PD-L1 resulted in a near-normal distribution, making the data more suitable for machine learning applications.}
    \label{fig: data preprocessing}
\end{figure}

\begin{table}[ht]
\small
\centering
\caption{This table summarizes previously reported generative models along with the datasets used for their training. All models were trained on datasets containing at least 25,000 molecules, with some reaching up to several million.}
\label{tab: large dataset table}
\begin{tabular}{lccc}
\toprule
Model          & Year & Dataset & Size \\
\midrule
MolGPT\cite{bagal2021molgpt}\hspace{20pt}          & 2021           & GuacaMol         & 1.1 Mil \\
JT-VAE\cite{jin2018junction}                        &2019          &PubChem           &10 Mil \\
GeoDiff\cite{xu2022geodiff}                          &2022         &QM9, QM8          &5 Mil\\
MEVO\cite{he2025generative}                    &2025            &REAL and ZINC20          &20 Mil\\
BioSNAP\cite{zhang2025zero}                   &2025                &Drug Bank           &27,464\\
TGM-DLM\cite{gong2024text}                    &2024               &ChEBI-20            &33,010\\
ACEGEN\cite{bou2024acegen}                    &2024              &ChEMBL and ZINC        &2 Mil\\
HYFORMER\cite{izdebski2025synergistic}         &2025              &GuacaMol          &1.1 Mil\\
\bottomrule
\end{tabular}
\end{table}

\subsection*{Structures of the new compounds}
Here, we present additional representative structures of the top-performing candidates, identified through in silico studies, from newly generated PD-L1 and kinase inhibitors.

\begin{figure}[H]
    \centering
    \includegraphics[width=1.0\linewidth]{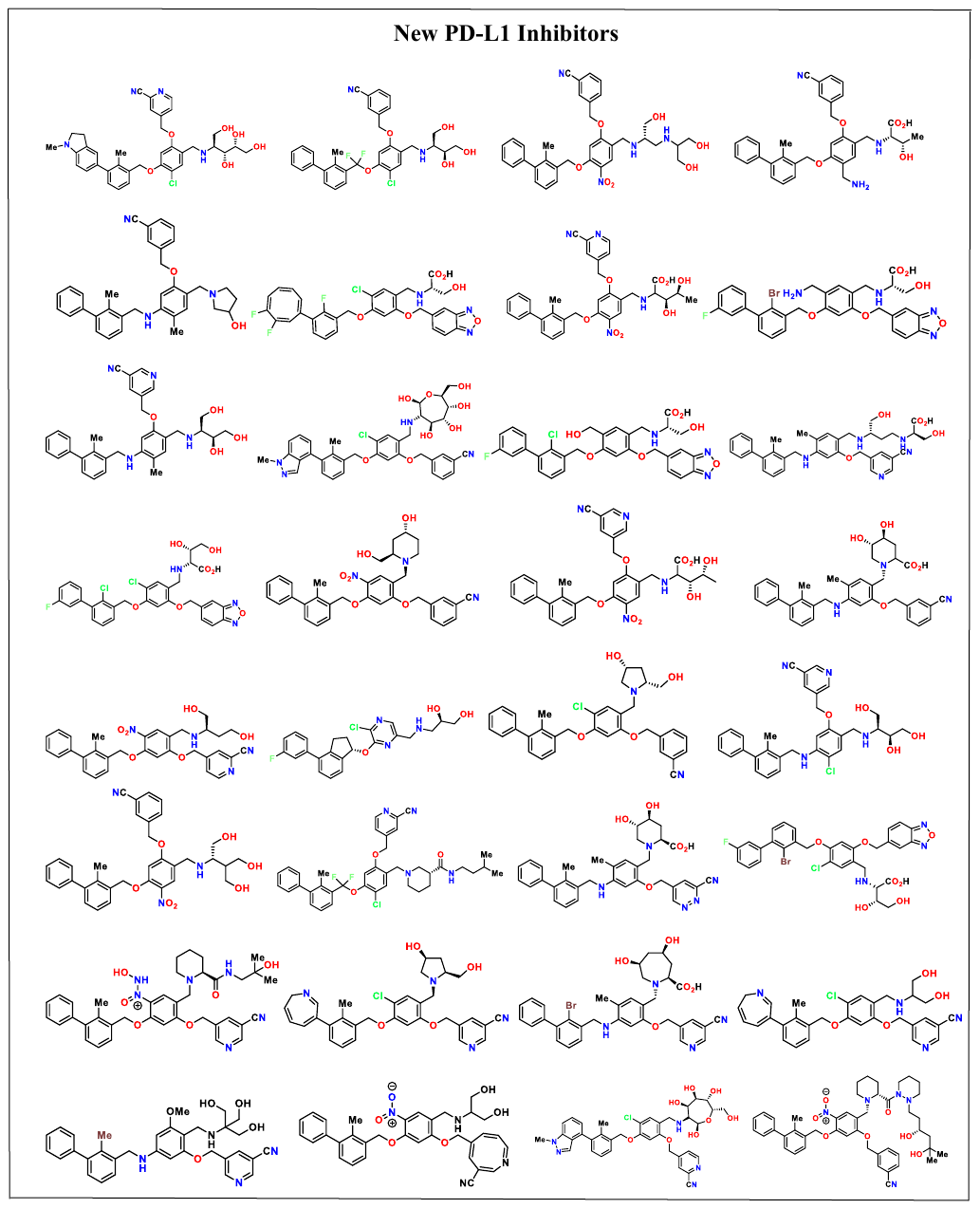}
    \caption{A total of 8,374 molecules were generated using our method for the PD-L1 dataset. In silico studies on these compounds identified several potential top hits. This figure presents the structures of a few representative examples.}
    \label{fig:SI_PDL1}
\end{figure}

\begin{figure}[H]
    \centering
    \includegraphics[width=1.0\linewidth]{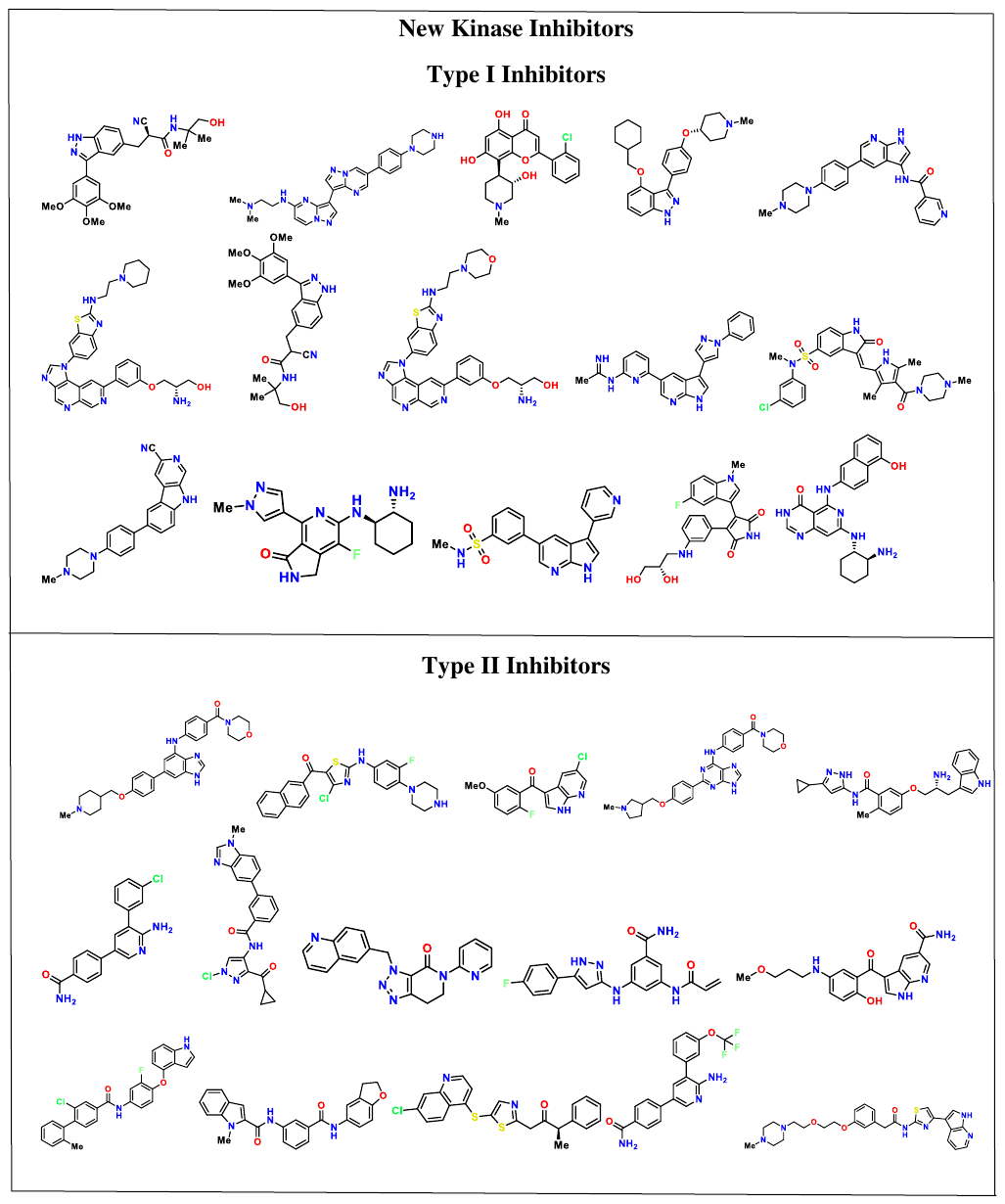}
    \caption{A total of 2500 molecules were generated using our method for each class of the Kinase dataset. In silico studies on these compounds identified several potential top hits. This figure presents the structures of a few representative examples for Type I and Type II classes.}
    \label{fig:SI_kinase_1}
\end{figure}

\begin{figure}[H]
    \centering
    \includegraphics[width=1.0\linewidth]{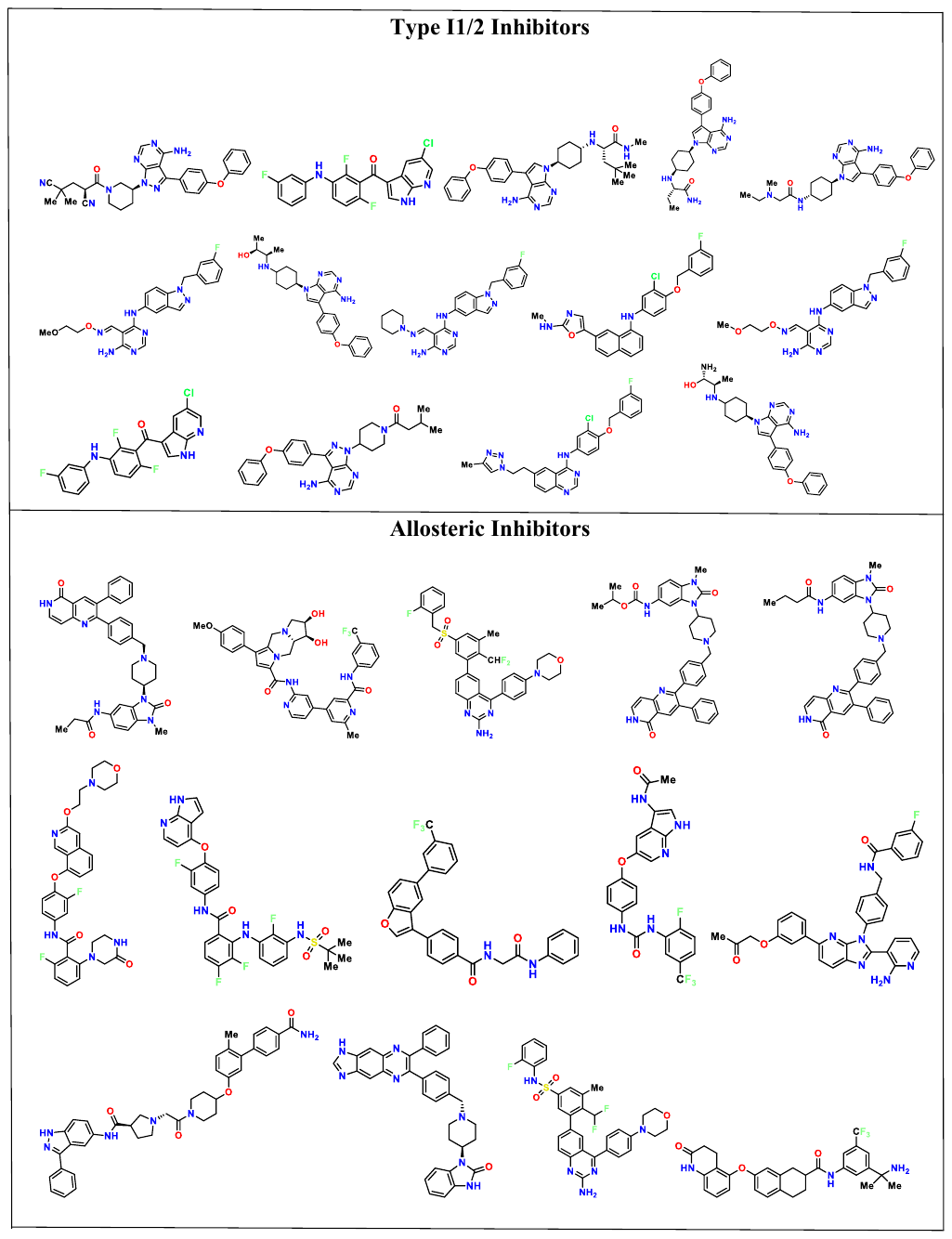}
    \caption{This figure presents the structures of a few representative examples for Type I\nicefrac{1}{2} and Allosteric inhibitors}
    \label{fig:SI_kinase_2}
\end{figure}

\subsection*{Model details}
\paragraph{Encoder Architecture:}
The encoder architecture is built upon a pre-trained ChemBERTa transformer model, which is used to generate rich molecular representations from SMILES strings. For each input SMILES, we extract the output corresponding to the [CLS] token from ChemBERTa's last hidden state, yielding a 768-dimensional embedding vector for the molecule. This base embedding is then passed through a non-linear projection head designed for contrastive learning. The projection head consists of a sequential module with two fully connected layers, (768 dimensions each) separated by an Exponential Linear Unit (ELU) activation function. This final projected embedding is used to structure the latent space according to molecular properties.

\paragraph{Decoder Architecture:}
We built a GRU-based decoder to convert latent vectors back into SMILES strings. The decoder included:
\begin{itemize}[topsep=1pt, itemsep=1pt, leftmargin=15pt]
    \item A linear layer to initialize the GRU hidden state.
    \item A character embedding layer for SMILES tokens.
    \item Three GRU layers with 512 hidden units and dropout.
    \item A final layer to output token probabilities.
\end{itemize}
We trained the decoder using teacher forcing and cross-entropy loss for 1000 epochs with the Adam optimizer (learning rate: 0.001).

\paragraph{Latent Sampling and SMILES Generation:}
After training the GMM, we sampled new latent vectors from each component. These vectors were decoded into SMILES strings using the GRU decoder with temperature-controlled softmax sampling. We used RDKit to remove invalid and duplicate molecules, keeping only unique and valid outputs. 

\subsection*{Other Downstream Tasks}
While contrastive SMILES embeddings can be used for a variety of tasks such as scaffold hopping, molecule classification and anomaly detection, we chose to focus on regression, i.e. prediction of $log(\text{IC}_{50})$ values as a critical downstream application. To rigorously assess the impact of contrastive learning, we conducted a comprehensive comparison between base ChemBERTa embeddings and our contrastively trained embeddings across seven different regression algorithms. To ensure robustness and statistical validity, each experiment was repeated over 50 independent random seeds. For each seed, we performed an $80\% - 20\%$ train-test split, maintaining identical data partitioning for both embedding types to enable direct performance comparisons. For the experiment with contrastive learning, the ChemBERTa model was independently trained for $1000$ epochs for each seed.

\begin{figure}[H]
  \centering
  \includegraphics[width=0.85\textwidth]{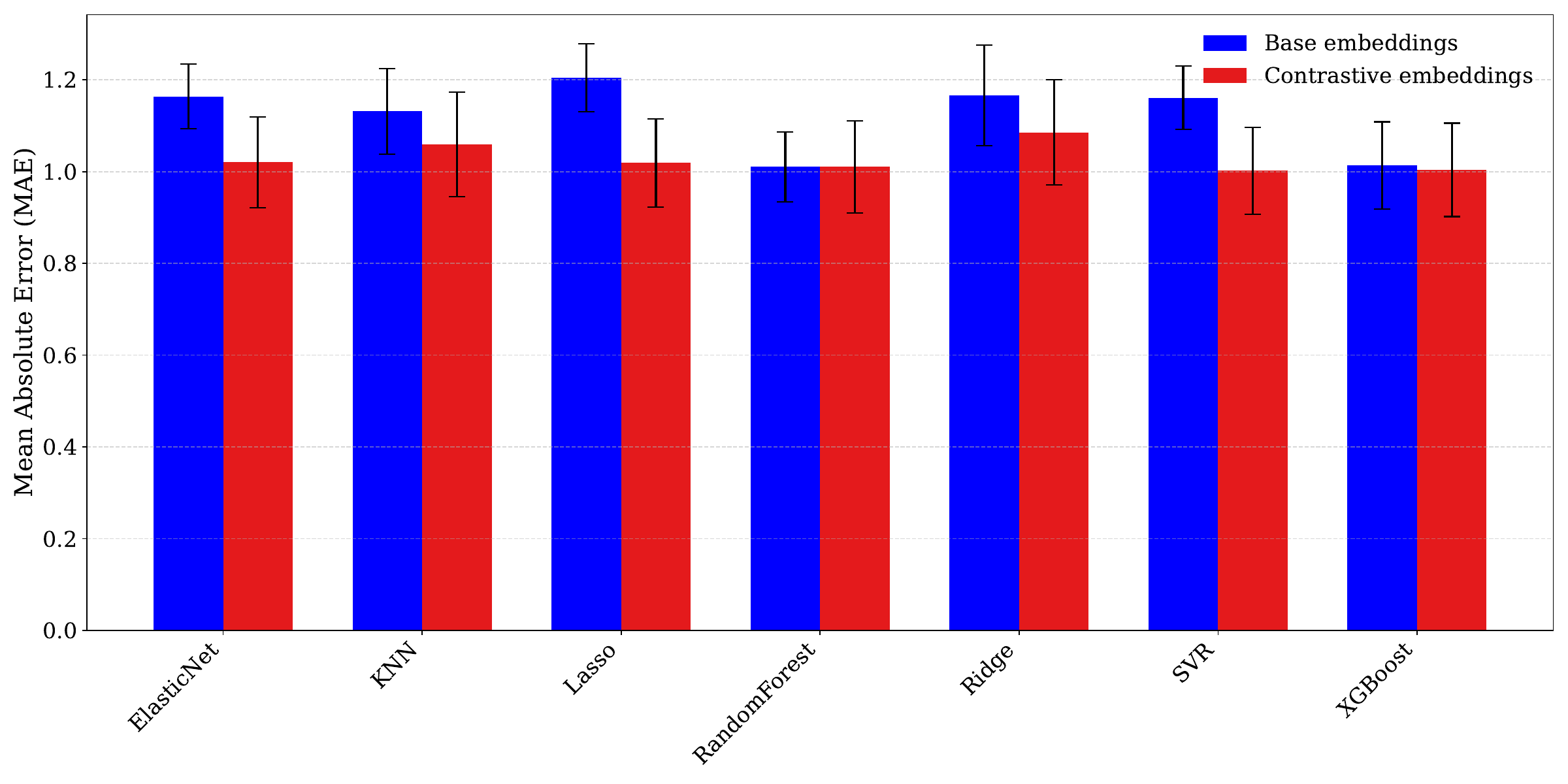}
  \caption{\textbf{MAE comparison across 7 regressors.}
    For each model, blue bars show the mean absolute error (MAE) using base embeddings,
    and red bars show MAE using contrastive embeddings; error bars denote one standard
    deviation over 50 seeds.}
  \label{fig:mae_comparison}
\end{figure}

\begin{table}[ht]
\centering
\caption{\textbf{Statistical analysis of MAE reductions.}
Paired two‐tailed $t$-tests and Cohen's $d$ were computed over 50 seeds to assess the
significance and effect size of MAE changes when switching from base to contrastive embeddings. Best values in each column are in \textbf{bold}.}
\label{tab:regression_stats}
\small
\begin{tabular}{lccccc}
\toprule
Model          & Avg. MAE\(_{\rm Base}\) $\downarrow$ & Avg. MAE\(_{\rm Contr.}\) $\downarrow$ & Avg. \(\Delta R^2 = (\Delta R_{\rm Contr.}^2 - \Delta R_{\rm Base}^2) \) $\uparrow$    & \(p\)-value      & Cohen’s \(d\)   \\
\midrule
Lasso          & 1.204654           & \textbf{1.019239}                  & \textbf{0.139012} & $<$1e–4 & {1.5088} \\
SVR            & 1.161183           & \textbf{1.001840}                  & \textbf{0.138585}          & $<$1e–4 & 1.5035          \\
ElasticNet     & 1.164035           & \textbf{1.020403}                  &\textbf{0.082211}          & $<$1e–4 & 1.2192          \\
Ridge          & 1.165921           & \textbf{1.085619}                  & \textbf{0.118779}          & 6.6e–5           & 0.6169          \\
k-NN           & 1.131383           & \textbf{1.059175}                  & \textbf{0.046120}          & 2.8e–4           & 0.5530          \\
XGBoost        & 1.013397           & \textbf{1.003876}         & –0.038543         & 0.5280           & 0.0899          \\
Random Forest  & \textbf{1.010385}  & 1.010401                  & –0.100013         & 0.9991           & –0.0002         \\
\bottomrule
\end{tabular}
\end{table}

Our results (Fig.~\ref{fig:mae_comparison}, Table~\ref{tab:regression_stats}) show that contrastive learning can substantially enhance prediction accuracy without requiring complex downstream models. Most notably, linear models and support vector regression exhibited the most substantial benefits. The Lasso regressor achieved the greatest improvement with a $15.4\%$ reduction in MAE (from $1.20$ to $1.02$) and a corresponding $R^2$ improvement of $0.139$, with a large effect size (Cohen's $d = 1.51$). Similarly, Support vector Regression showed a $13.7\%$ reduction in MAE with comparable $R^2$ gains. ElasticNet, Ridge, and k-NN models also demonstrated statistically significant improvements (all $p < 0.001$), although with progressively smaller effect sizes. Interestingly, tree-based ensemble methods (Random Forest and XGBoost) did not benefit from contrastive embeddings, showing no statistically significant differences in performance. These findings underscore that contrastive SMILES training is especially valuable for models with limited built-in nonlinearity or feature selection, while very expressive learners achieve near‐optimal performance regardless. Such insights are consistent with work such as \cite{banerjee2023surprisal} where diverse set of regression methods were benchmarked against a large collection of datasets. Such insights will guide future work on matching embedding strategies to downstream model complexity across diverse cheminformatics tasks.

\end{document}